\documentclass[11pt]{article}

\usepackage{fullpage,times,url}

\usepackage{amsthm,amsfonts,amsmath,amssymb,epsfig,color,float,graphicx,verbatim}
\usepackage{algorithm,algorithmic}

\usepackage{hyperref}
\hypersetup{
	colorlinks   = true, 
	urlcolor     = blue, 
	linkcolor    = blue, 
	citecolor   = red 
}

\newtheorem{theorem}{Theorem}
\newtheorem{proposition}{Proposition}
\newtheorem{lemma}{Lemma}

\newtheorem{definition}{Definition}
\newtheorem{remark}{Remark}

\newcommand{\reals}{\mathbb{R}}
\newcommand{\E}{\mathbb{E}}

\newcommand{\var}{\text{Var}}

\newcommand{\bx}{\mathbf{x}}
\newcommand{\bw}{\mathbf{w}}
\newcommand{\bg}{\mathbf{g}}

\newcommand{\bv}{\mathbf{v}}
\newcommand{\bz}{\mathbf{z}}

\newcommand{\bp}{\mathbf{p}}

\newcommand{\bxi}{\boldsymbol{\xi}}

\newcommand{\Ocal}{\mathcal{O}}
\newcommand{\Acal}{\mathcal{A}}

\newcommand{\Dcal}{\mathcal{D}}

\newcommand{\Hcal}{\mathcal{H}}

\newcommand{\Wcal}{\mathcal{W}}

\newcommand{\norm}[1]{\left\|#1\right\|}
\newcommand{\inner}[1]{\left\langle#1\right\rangle}

\newcommand{\secref}[1]{Sec.~\ref{#1}}

\renewcommand{\eqref}[1]{Eq.~(\ref{#1})}
\newcommand{\lemref}[1]{Lemma~\ref{#1}}

\newcommand{\thmref}[1]{Thm.~\ref{#1}}

\title{Distribution-Specific Hardness of Learning Neural Networks}
\author{Ohad Shamir\\Weizmann Institute of 
Science\\\texttt{ohad.shamir@weizmann.ac.il}}
\date{}

\begin{document}
	
	\maketitle
	
	\begin{abstract}
		Although neural networks are routinely and successfully trained in 
		practice using simple gradient-based methods, most existing theoretical 
		results are negative, showing that learning such networks is difficult, 
		in a worst-case sense over all data distributions. In this paper, we 
		take a more nuanced view, and consider whether specific assumptions on 
		the ``niceness'' of the input distribution, or ``niceness'' of the 
		target function (e.g. in terms of smoothness, non-degeneracy, 
		incoherence, random choice of parameters etc.), are sufficient to 
		guarantee learnability using gradient-based methods. We provide 
		evidence that neither class of assumptions alone is sufficient: On the 
		one hand, for any 
		member of a class of ``nice'' target functions, there are 
		difficult input distributions. On the other hand, we identify a 
		family of simple target functions, which are difficult to learn even if 
		the input distribution is ``nice''. To prove our results, 
		we develop some tools which may be of independent interest, such as 
		extending Fourier-based hardness techniques developed in the context of 
		statistical queries \cite{blum1994weakly}, from the Boolean 
		cube to Euclidean space and to more general classes of functions.
	\end{abstract}
	
	\section{Introduction}
	
	Artificial neural networks have seen a dramatic resurgence in recent years, and have proven to be a highly effective machine learning method in computer vision, natural language processing, and other challenging AI problems. Moreover, successfully training such networks is routinely performed using simple and scalable gradient-based methods, in particular stochastic gradient descent. 
	
	Despite this practical success, our theoretical understanding of the computational tractability of such methods is quite limited, with most results being negative. For example, as discussed in \cite{livni2014computational}, learning even depth-2 networks in a formal PAC learning framework is computationally hard in the worst case, and even if the algorithm is allowed to return arbitrary predictors. As common in such worst-case results, these are proven using rather artificial constructions, quite different than the real-world problems on which neural networks are highly successful. In particular, since the PAC framework focuses on \emph{distribution-free} learning (where the distribution generating the examples is unknown and rather arbitrary), the hardness results rely on carefully crafted distributions, which allows one to relate the learning problem to (say) an NP-hard problem or breaking a cryptographic system. However, what if we insist on ``natural'' distributions? Is it possible to show that  neural networks learning becomes computationally tractable? Can we show that they can be learned using the standard heuristics employed in practice, such as stochastic gradient descent?
	
	To understand what a ``natural'' distribution refers to, we need to separate the distribution over examples (given as input-output pairs $(\bx,y)$) into two components:
	\begin{itemize}
		\item \emph{The input distribution $p(\bx)$}: ``Natural'' input distributions on Euclidean space tend to have properties such as smoothness, non-degeneracy, incoherence etc. 
		\item \emph{The target function $h(\bx)$}: In PAC learning, it is assumed that the output $y$ equals $h(\bx)$, where $h$ is some unkown target function from the hypothesis class we are considering. In studying neural networks, it is common to consider the class of all networks which share some fixed architecture (e.g. feedforward networks of a given depth and width). However, one may argue that the parameters of real-world networks (e.g. the weights of each neuron) are not arbitrary, but exhibit various features such as non-degeneracy or some  ``random like'' appearance. Indeed, networks with a random structure have been shown to be more amenable to analysis in various situations (see for instance \cite{daniely2016toward,arora2014provable,choromanska2015loss} and references therein). 
	\end{itemize}
	Empirical evidence clearly suggest that many pairs of input distributions and target functions are computationally tractable, using standard methods. However, how do we characterize these pairs? Would appropriate assumptions on one of them be sufficient to show learnability?
	
	In this paper, we investigate these two components, and provide evidence that \emph{neither one} of them alone is enough to guarantee computationally tractable learning, at least with methods resembling those used in practice. Specifically, we focus on simple, shallow ReLU networks, assume that the data can be perfectly predicted by some such network, and even allow \emph{over-specification} (a.k.a. improper learning), in the sense that we allow the learning algorithm to output a predictor which is possibly larger and more complex than the target function (this technique increases the power of the learner, and was shown to make the learning problem easier in theory and in practice \cite{livni2014computational,SafranS16,soudry2016no}). Even under such favorable conditions, we show  the following:
	\begin{itemize}
		\item \textbf{Hardness for ``natural'' target functions.} For each individual target function coming from a simple class of small, shallow ReLU networks (even if its parameters are chosen randomly or in some other oblivious way), we show that no algorithm invariant to linear transformations can successfully learn it w.r.t. all input distributions in polynomial time (this corresponds, for instance, to standard gradient-based methods together with data whitening or preconditioning). This result is based on a reduction from learning intersections of halfspaces. Although that problem is known to be hard in the worst-case over both input distributions and target functions, we essentially show that invariant algorithms as above do not ``distinguish'' between worst-case and average-case: If one can learn a particular target function with such an algorithm, then the algorithm can learn nearly all target functions in that class.
		
		\item \textbf{Hardness for ``natural'' input distributions.} We show 
		that target functions of the form $\bx\mapsto \psi(\inner{\bw,\bx})$ 
		for any periodic $\psi$ are generally difficult to learn using  
		gradient-based methods, even if the input distribution is fixed and 
		belongs to a very broad class of smooth input distributions (including, 
		for instance, Gaussians and mixtures of Gaussians). Note that such 
		functions can be constructed by simple shallow networks, and can be 
		seen as an extension of generalized linear models 
		\cite{mccullagh1989generalized}. Unlike the previous result, which 
		relies on a computational hardness assumption, the results here are
		geometric in nature, and imply that the gradient of the objective 
		function, nearly everywhere, contains virtually no signal on the 
		underlying target function. Therefore, any algorithm which relies on 
		gradient information cannot learn such functions. Interestingly, the 
		difficulty here is \emph{not} in having a plethora of spurious local 
		minima or saddle points -- the associated stochastic optimization 
		problem may actually have no such critical points. Instead, the 
		objective function may exhibit properties such as \emph{flatness} 
		nearly everywhere, unless one is already very close to the global 
		optimum. This highlights a potential pitfall in non-convex learning, 
		which occurs already for a slight extension of generalized linear 
		models, and even for ``nice'' input distributions. 
	\end{itemize}
	Together, these results indicate that in order to explain the practical success of neural network learning with gradient-based methods, one would need to employ a careful combination of assumptions on both the input distribution and the target function, and that results with even a ``partially'' distribution-free flavor (which are common, for instance, in convex learning problems) may be difficult to attain here.
	
	To prove our results, we develop some tools which may be of independent interest. In particular, the techniques used to prove hardness of learning functions of the form $\bx\mapsto \psi(\inner{\bw,\bx})$ are based on Fourier analysis, and have some close connections to hardness results on learning parities in the well-known framework of learning from statistical queries  \cite{kearns1998efficient}: In both cases, one essentially shows that the Fourier transform of the target function has very small support, and hence does not ``correlate'' with most functions, making it difficult to learn using certain methods. However, we consider a more general and arguably more natural class of input distributions over Euclidean space, rather than distributions on the Boolean cube. In a sense, we show that learning general periodic functions over Euclidean space is difficult (at least with gradient-based methods), for the same reasons that learning parities over the Boolean cube is difficult in the statistical queries framework.

	\subsection*{Related Work}
	
	Recent years have seen quite a few papers on the theory of neural network learning. Below, we only briefly mention those most relevant to our paper.
	
	In a very elegant work, Janzamin et al. \cite{janzamin2015beating} have shown that a certain method based on tensor decompositions allows one to provably learn simple neural networks by a combination of assumptions on the input distribution and the target function. However, a drawback of their method is that it requires rather precise knowledge of the input distribution and its derivatives, which is rarely available in practice. In contrast, our focus is on algorithms which do not utilize such knowledge. Other works which show computationally-efficient learnability of certain neural networks under sufficiently strong distributional assumptions include \cite{arora2014provable,livni2014computational,andoni2014learning,zhang2015learning}. 
	
	In the context of learning functions over the Boolean cube, it is known 
	that even if we restrict ourself to a particular input distribution (as 
	long as it satisfies some mild conditions), it is difficult to learn parity 
	functions using statistical query algorithms 
	\cite{kearns1998efficient,blum1994weakly}. Moreover, it was recently shown 
	that stochastic gradient descent methods can be approximately posed as such 
	algorithms \cite{feldman2015statistical}. Since parities can be implemented 
	with small real-valued networks, this implies that for ``most'' input 
	distributions on the Boolean cube, there are neural networks which are 
	unlikely to be learnable with gradient-based methods. However, data 
	provided to 
	neural networks in practice are not in the form of Boolean vectors, but 
	rather vectors of floating-point numbers. Moreover, some assumptions on the 
	input distribution, such as smoothness and Gaussianity, only make sense 
	once we consider the support to be Euclidean space rather than Boolean 
	cube. Perhaps these are enough to guarantee computational tractability? A 
	contribution of this paper is to show that this is not the case, and to 
	formally demonstrate how phenomena similar to the Boolean case also occurs 
	in Euclidean space, using appropriate target functions and distributions. 
	
	Finally, we note that \cite{klko14} provides improper-learning hardness results, which hold even for a standard Gaussian distribution on Euclidean space, and for any algorithm. However, unlike our paper, their focus is on hardness of agnostic learning (where the target function is arbitrary and does not have to correspond to a given class), the results are specific to the standard Gaussian distribution, and the proofs are based on a reduction from the Boolean case.
	
	The paper is structured as follows: In \secref{sec:prelim}, we formally present some notation and concepts used throughout the paper. In \secref{sec:input}, we provide our hardness results for natural input distributions, and in \secref{sec:target}, we provide our hardness results for natural target functions. All proofs are presented in \secref{sec:proofs}.

	\section{Preliminaries}\label{sec:prelim}
	
	We generally let bold-faced letters denote vectors. Given a complex-valued number $z=a+ib$, we let $\overline{z}=a-ib$ denote its complex conjugate, and $|z|=\sqrt{a^2+b^2}$ denote its modulus.  
	Given a function $f$, we let $\nabla f$ denote its gradient and $\nabla^2 f$ denote its Hessian (assuming they exist). 
	
	\textbf{Neural Networks.} The focus of our results will be on learning predictors which can be described by simple and shallow (depth 2 or 3) neural networks. A standard feedforward neural network is composed of neurons, each of which computes the mapping $\bx\mapsto \sigma(\inner{\bw,\bx}+b)$, where $\bw,b$ are parameters and $\sigma$ is a scalar activation function, for example the popular ReLU function $[z]_+ = \max\{0,z\}$. These neurons are arranged in parallel in layers, so the output of each layer can be compactly represented as $\bx\mapsto \sigma(W^\top\bx+\mathbf{b})$, where $W$ is a matrix (each column corresponding to the parameter vector of one of the neurons), $\mathbf{b}$ is a vector, and $\sigma$ applies an activation function on the coordinates of $W^\top\bx$. In vanilla feedforward networks, such layers are connected to each other, so given an input $\bx$, the output equals
	\[
	\sigma_k(W_k^\top\sigma_{k-1}(W_{k-1}^\top\ldots \sigma_2(W_2^\top\sigma_1(W_1^\top\bx+\mathbf{b}_1)+\mathbf{b}_2)\ldots+\mathbf{b}_{k-1})+\mathbf{b}_k),
	\]
	where $W_i,b_i,\sigma_i$ are parameter of the $i$-th layer. The number of layers $k$ is denoted as the depth of the network, and the maximal number of columns in $W_i$ is denoted as the width of the network. For simplicity, in this paper we focus on networks which output a real-valued number, and measure our performance with respect to the squared loss (that is, given an input-output example $(\bx,y)$, where $\bx$ is a vector and $y\in \reals$, the loss of a predictor $p$ on the example is $(p(\bx)-y)^2$). 
	
	\textbf{Gradient-Based Methods.} Gradient-based methods are a class of optimization algorithms for solving problems of the form $\min_{\bw\in\Wcal} F(\bw)$ (for some given function $F$ and assuming $\bw$ is a vector in Euclidean space), based on computing $\nabla F(\bw)$ of approximations of $\nabla F(\bw)$ at various points $\bw$. Perhaps the simplest such algorithm is gradient descent, which initializes deterministically or randomly at some point $\bw_1$, and iteratively performs updates of the form $\bw_{t+1} = \bw_t-\eta_t \nabla F(\bw_t)$, where $\eta_t>0$ is a step size parameter. In the context of statistical supervised learning problems, we are usually interested in solving problems of the form $\min_{\bw\in\Wcal} \E_{\bx\sim\Dcal}[\ell(f(\bw,\bx),h(\bx))]$, where $\{\bx\mapsto f(\bw,\bx):\bw\in\Wcal\}$ is some class of predictors, $h$ is a target function, and $\ell$ is some loss function. Since the distribution $\Dcal$ is generally unknown, one cannot compute the gradient of this function w.r.t. $\bw$ directly, but can still compute approximations, e.g. by sampling one $\bx$ at random and computing the gradient (or sub-gradient) of $\ell(f(\bw,\bx),h(\bx))$. The same approach can be used to solve empirical approximations of the above, i.e. $\min_{\bw\in\Wcal} \frac{1}{m}\sum_{i=1}^{m}\ell(f(\bw,\bx_i),h(\bx_i))$ for some dataset $\{(\bx_i,h(\bx_i))\}_{i=1}^{m}$. These are generally known as stochastic gradient methods, and are one of the most popular and scalable machine learning methods in practice.
	
	\textbf{PAC Learning.} For the results of \secref{sec:target}, we will rely on the following standard definition of PAC learning with respect to Boolean functions: Given a hypothesis class $\Hcal$ of functions from $\{0,1\}^d$ to $\{0,1\}$, we say that a learning algorithm PAC-learns $\Hcal$ if for any $\epsilon\in (0,1)$, any distribution $\Dcal$ over $\{0,1\}^d$, and any $h^\star\in \Hcal$, if the algorithm is given oracle access to i.i.d. samples $(\bx,h^\star(\bx))$ where $\bx$ is sampled according to $\Dcal$, then in time $\text{poly}(d,1/\epsilon)$, the algorithm returns a function $f:\{0,1\}^d\mapsto \{0,1\}$ such that $\Pr_{\bx\sim\Dcal}(f(\bx)\neq h^\star(\bx))\leq \epsilon$ with high probability (for our purposes, it will be enough to consider any constant close to $1$).
	Note that in the definition above, we allow $f$ not to belong to the hypothesis class $\Hcal$. This is often denoted as ``improper'' learning, and allows the learning algorithm more power than in ``proper'' learning, where $f$ must be a member of $\Hcal$.
		
	\textbf{Fourier Analysis on $\reals^d$.} In the analysis of \secref{sec:input}, we will consider functions from $\reals^d$ to the reals $\reals$ or complex numbers $\mathbb{C}$, and view them as elements in the Hilbert space $L^2(\reals^d)$ of square integrable functions, equipped with the inner product
	\[
	\inner{f,g} = \int_{\bx}f(\bx)\cdot\overline{g(\bx)}d\bx
	\]
	and the norm $\norm{f} = \sqrt{\inner{f,f}}$. We use $fg$ or $f\cdot g$ as shorthand for the function $\bx\mapsto f(\bx)g(\bx)$. Any function $f\in L^2(\reals^d)$ has a Fourier transform $\hat{f}\in L^2(\reals^d)$, which for absolutely integrable functions can be defined as 
	\begin{equation}\label{eq:fourierdef}
	\hat{f}(\bw) = \int \exp\left(-2\pi i\inner{\bx,\bw}\right)f(\bx)d\bx,
	\end{equation}
	where $\exp(iz) = \cos(z)+i\cdot \sin(z)$, $i$ being the imaginary unit. In the proofs, we will use the following well-known properties of the Fourier transform:
	\begin{itemize}
		\item Linearity: For scalars $a,b$ and functions $f,g$, $\widehat{af+bg} = a\hat{f}+b\hat{g}$.
		\item Isometry: $\inner{f,g} = \langle\hat{f},\hat{g}\rangle$ and $\norm{f}=\|\hat{f}\|$.
		\item Convolution: $\widehat{fg} = \hat{f}*\hat{g}$, where $*$ denotes the convolution operation: $(f*g)(\bw) = \int f(\bz)\cdot g(\bw-\bz)~d\bz$. 
		\end{itemize}

	\section{Natural Target Functions}\label{sec:target}
	
	In this section, we consider simple target functions of the form 	$\bx\mapsto \left[\sum_{i=1}^{n}[\inner{\bw_i,\bx}]_+\right]_{[0,1]}$, where $[z]_+ = \max\{0,z\}$ is the ReLU function, and $[z]_{[0,1]} = \min\{1,\max\{0,z\}\}$ is the clipping operation on the interval $[0,1]$. This corresponds to depth-2 networks with no bias in the first layer, and where the outputs of the first layer are simply summed and moved through a clipping non-linearity (this operation can also be easily implemented using a second layer composed of two ReLU neurons). Letting $W = [\bw_1,\ldots,\bw_n]$, we can write such predictors as $\bx\mapsto h(W^\top \bx)$ for an appropriate fixed function $h$. Our goal would be to show that for such a target function, with virtually any choice of $W$ (essentially, as long as its columns are linearly independent), and any polynomial-time learning algorithm satisfying some conditions, there exists an input distribution on which it must fail.
	
	As the careful reader may have noticed, it is impossible to provide such a target-function-specific result which holds for any algorithm. Indeed, if we fix the target function in advance, we can always ``learn'' by returning the target function, regardless of the training data. Thus, imposing some constraints on the algorithm is necessary. Specifically, we will consider algorithms which exhibit certain natural invariances to the coordinate system used. One very natural invariance is with respect to orthogonal transformations: For example, if we rotate the input instances $\bx_i$ in a fixed manner, then an orthogonally-invariant algorithm will return a predictor which still makes the same predictions on those instances. Formally, this invariance is defined as follows:
	\begin{definition}\label{def:uninv}
		Let $\Acal$ be an algorithm which inputs a dataset $(\{\bx_i,y_i\})_{i=1}^m$ (where $\bx_i\in \reals^d$) and outputs a predictor $\bx\mapsto f(W^\top\bx)$ (for some function $f$ and matrix $W$ dependent on the dataset). We say that $\Acal$ is \emph{orthogonally-invariant}, if for any orthogonal matrix $M\in \reals^{d\times d}$, if we feed the algorithm with $\{M\bx_i,y_i\}_{i=1}^{m}$, the algorithm returns a predictor $\bx\mapsto f(W_M^\top \bx)$, where $f$ is the same as before and $W_M$ is such that $W_M^\top M \bx_i= W^\top \bx_i$ for all $\bx_i$. 
	\end{definition}
	\begin{remark}
		The definition as stated refers to deterministic algorithms. For stochastic algorithms, we will understand orthogonal invariance to mean orthogonal invariance conditioned on any realization of the algorithm's random coin flips.
	\end{remark}
	
	For example, standard gradient and stochastic gradient descent methods (possibly with coordinate-oblivious regularization, such as $L_2$ regularization) can be easily shown to be orthogonally-invariant\footnote{Essentially, this is because the gradient of any function $g(W^\top \bx)=g(\inner{\bw_1,\bx},\ldots,\inner{\bw_k,\bx})$ w.r.t. any $\bw_i$ is proportional to $\bx$. Thus, if we multiply $\bx$ by an orthogonal $M$, the gradient also gets multiplied by $M$. Since $M^\top M=I$, the inner products of instances $\bx$ and gradients remain the same. Therefore, by induction, it can be shown that any algorithm which operates by incrementally updating some iterate by linear combinations of gradients will be rotationally invariant.}. However, for our results we will need to make a somewhat stronger invariance assumption, namely invariance to general invertible linear transformations of the data (not necessarily just orthogonal). This is formally defined as follows:
	
	\begin{definition}\label{def:afinv}
		An algorithm $\Acal$ is \emph{linearly-invariant}, if it satisfies Definition \ref{def:uninv} for any invertible matrix $M\in \reals^{d\times d}$ (rather than just orthogonal ones).
	\end{definition}
	
	One well-known example of such an algorithm (which is also invariant to affine transformations) is the Newton method \cite{boyd2004convex}. More relevant to our purposes,  linear invariance occurs whenever an orthogonally-invariant algorithm preconditions or ``whitens'' the data so that its covariance has a fixed structure (e.g. the identity matrix, possibly after a dimensionality reduction if the data is rank-deficient). For example, even though gradient descent methods are not linearly invariant, they become so if we precede them by such a preconditioning step. This is formalized in the following theorem:
	
	\begin{theorem}\label{thm:affuninv}
		Let $\Acal$ be any algorithm which given $\{\bx_i,y_i\}_{i=1}^{m}$, computes the whitening matrix $P = D^{-1}U^\top$ (where $X=[\bx_1~\bx_2~\ldots~\bx_m]$, $X=U D V^\top$ is a thin\footnote{That is, if $X$ is of size $d\times m$, then $U$ is of size $d\times \text{Rank}(X)$, $D$ is of size $\text{Rank}(X)\times\text{Rank}(X)$, and $V$ is of size $m\times \text{Rank}(X)$.} SVD decomposition of $X$), feeds $\{P\bx_i,y_i\}_{i=1}^{m}$ to an orthogonally-invariant algorithm, and given the output predictor $\bx\mapsto f(W^\top\bx)$, returns the predictor $\bx\mapsto f((P^\top W)^\top \bx)$. Then $\Acal$ is linearly-invariant.
	\end{theorem}
	
	It is easily verified that the covariance matrix of the transformed 
	instances $P \bx_1,\ldots,P\bx_m$ is the $r\times r$ identity matrix (where 
	$r=\text{Rank}(X)$), so this is indeed a whitening transform. We note that 
	whitening is a very common preprocessing heuristic, and even when not done 
	explicitly, scalable approximate whitening and preconditioning methods 
	(such as Adagrad \cite{duchi2011adaptive} and batch normalization 
	\cite{ioffe2015batch}) are very common and widely recognized as useful for 
	training neural networks. 
	
	To show our result, we rely on a reduction from a PAC-learning problem known to be computationally hard, namely learning intersections of halfspaces. These are Boolean predictors parameterized by  $\bw_1,\ldots,\bw_n\in \reals^d$ and $b_1,\ldots,b_n\in \reals$, which compute a mapping of the form
	\[
	\bx ~\to~\bigwedge_{i=1}^{n}\left(\inner{\bw_i,\bx}\geq b_i\right)
	\]
	(where we let $1$ correspond to `true' and $0$ to `false'). 
	The problem of PAC-learning intersections of halfspaces over the Boolean cube ($\bx\in\{0,1\}^d$) has been well-studied. In particular, two known hardness results are the following:
	\begin{itemize}
		\item Klivans and Sherstov \cite{klivans2009cryptographic} show that under a certain well-studied cryptographic assumption (hardness of finding unique shortest vectors in a high-dimensional lattice), no algorithm can PAC-learn intersection of $n_d=d^{\delta}$ halfspaces (where $\delta$ is any positive constant), even if the coordinates of $\bw_i$ and $b_i$ are all integers, and $\max_i \norm{(\bw_i,b_i)} \leq \text{poly}(d)$.
		\item Daniely and Shalev-Shwartz \cite{daniely2016complexity} show that under an assumption related to the hardness of refuting random K-SAT formulas, no algorithm can PAC-learn intersections of $n_d=\omega(1)$ halfspaces (as $d\rightarrow \infty$), even if the coordinates of $\bw_i$ and $b_i$ are all integers, and $\max_i \norm{(\bw_i,b_i)}\leq \Ocal(d)$. 
	\end{itemize}
	In the theorem below, we will use the result of \cite{daniely2016complexity}, which applies to an intersection of a  smaller number of halfspaces, and with smaller norms. However, similar results can be shown using \cite{klivans2009cryptographic}, at the cost of worse polynomial dependencies on $d$.

	The main result of this section is the following:
	\begin{theorem}\label{thm:maintarget}
		Consider any network $h(W_\star^\top \bx)=\left[\sum_{i=1}^{n_d}[\inner{\bw^\star_i,\bx}]_+\right]_{[0,1]}$ (where $W_\star = [\bw^\star_1,\ldots,\bw^\star_n]$), which satisfies the following:
		\begin{itemize}
			\item $n_d \geq  \omega(1)$ as $d\rightarrow \infty$
			\item $\max_i \norm{\bw^\star_i}\leq \Ocal(d)$
			\item $\bw_1^{\star}\ldots\bw_n^{\star}$ are linearly independent, so the smallest singular value $s_{\min}(W_\star)$ of $W_\star$ is strictly positive.
		\end{itemize}
		Then under the assumption stated in \cite{daniely2016complexity},   
		there is no linearly-invariant algorithm which for any $\epsilon>0$ and 
		any distribution $\Dcal$ over vectors of norm at most 
		$\frac{\Ocal(d\sqrt{dn_d})}{\min\{1,s_{\min}(W_\star)\}}$, given only 
		access to samples $(\bx,h(W_\star^\top \bx))$ where $\bx\sim \Dcal$, 
		runs in time $\text{poly}(d,1/\epsilon)$ and returns with high 
		probability a predictor $\bx\mapsto f(W^\top \bx)$ such that 
		\[
		\E_{\bx\sim \Dcal}\left[\left(f(W^\top\bx)-h(W_\star^\top \bx)\right)^2\right]\leq \epsilon.
		\]
	\end{theorem}
	Note that the result holds even if the returned predictor $f(W^\top \bx)$ has a different structure than $h_{W_\star}(\cdot)$, and $W$ is of a larger size than $W_\star$. Thus, it applies even if the algorithm is allowed to train a larger network or more complicated predictor than $h_{W_\star}(\cdot)$. 
	
	The proof (which is provided in \secref{sec:proofs}) can be sketched as follows: First, the hardness assumption for learning intersection of halfspaces is shown to imply hardness of learning networks $\bx\mapsto h(W^\top \bx)$ as described above (and even if $W$ has linearly independent columns -- a restriction which will be important later). However, this only implies that no algorithm can learn $\bx\mapsto h(W^\top \bx)$ for \emph{all} $W$ and all input distributions $\Dcal$. In contrast, we want to show that learning would be difficult even for \emph{some fixed} $W_\star$. To do so, we show that if an algorithm is linearly invariant, then the ability to learn with respect to some $W$ and all distributions $\Dcal$ means that we can learn with respect to all $W$ and all $\Dcal$. Roughly speaking, we argue that for linearly-invariant algorithms, ``average-case'' and ``worst-case'' hardness are the same here. Intuitively, this is because given some arbitrary $W,\Dcal$, we can create a different input distribution $\tilde{\Dcal}$, so that $W,\tilde{\Dcal}$ ``look like'' $W_\star,\Dcal$ under some linear transformation (see Figure \ref{fig:trans} for an illustration). Therefore, a linearly-invariant algorithm which succeeds on one will also succeed on the other.
	
	\begin{figure}[t]
		\centering
		\includegraphics[trim=0cm 0.2cm 0cm 0.6cm, clip=true,scale=0.7]{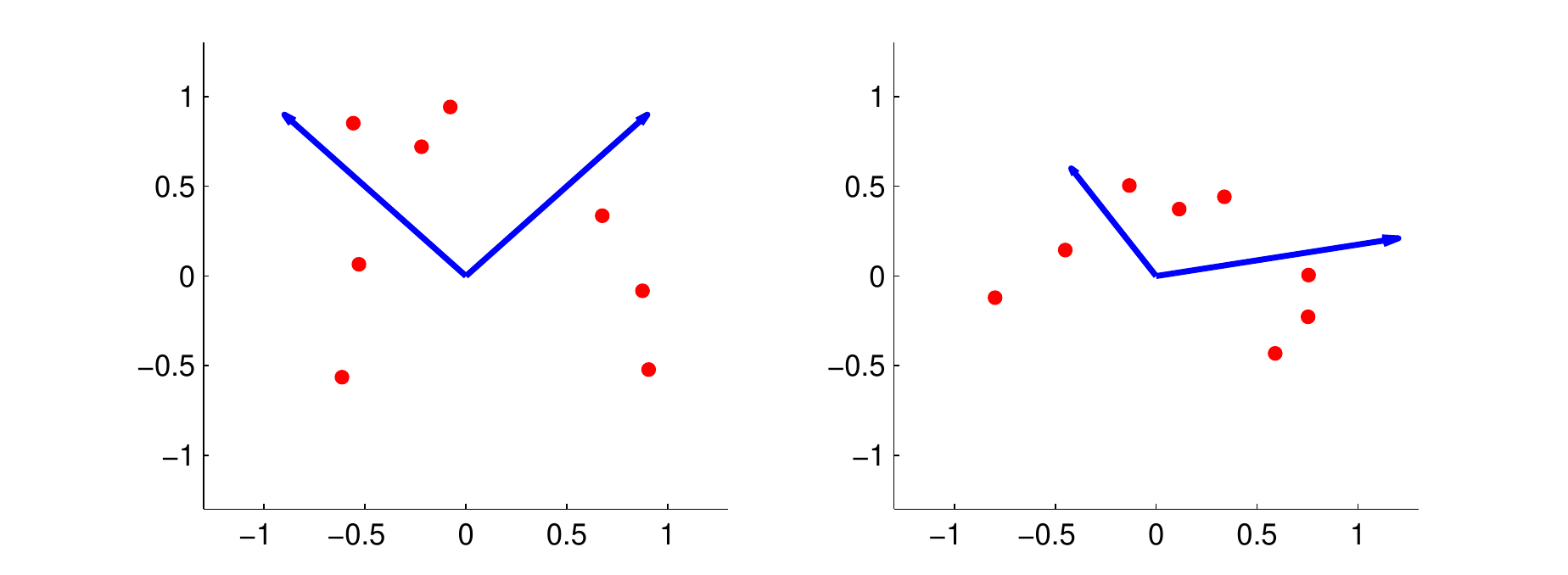}
		\caption{Correspondence between $W_\star,\Dcal$ (left figure) and $W,\tilde{\Dcal}$ (right figure). Arrows correspond to columns of $W_\star$ and $W$, and dots correspond to the support of $\Dcal$ and $\tilde{\Dcal}$. $\tilde{\Dcal}$ is constructed so that the same linear transformation mapping $W_\star$ to $W$ also maps $\Dcal$ to $\tilde{\Dcal}$. }\label{fig:trans}
	\end{figure}

	A bit more formally, let us fix some $W_\star$ (with linearly independent columns), and suppose we have a linearly-invariant algorithm which can successfully learn $\bx\mapsto h(W_\star^\top \bx)$ with respect to any input distribution. Let $W,\Dcal$ be some other matrix and distribution with respect to which we wish to learn (where $W$ has full column rank and is of the same size as $W_\star$). Then it can be shown that there is an invertible matrix $M$ such that $W=M^\top W_\star$. Since the algorithm successfully learns $\bx\mapsto h(W_\star^\top \bx)$ with respect to any input distribution, it would also successfully learn if we use the input distribution $\tilde{\Dcal}$ defined by sampling $\bx\sim \Dcal$ and returning $M\bx$. This means that the algorithm would succesfully learn from data distributed as
	\[
	(\bx,h(W_\star^\top\bx)~,~\bx\sim\tilde{\Dcal}~~\Longleftrightarrow~~(M\bx,h(W_\star^\top(M\bx)))~,~\bx\sim\Dcal~~\Longleftrightarrow~~(M\bx,h(W^\top \bx))~,~\bx\sim\Dcal.
	\] 
	Since the algorithm is linearly-invariant, it can be shown that this implies successful learning from $(\bx,h(W^\top \bx))$ where $\bx\sim\Dcal$, as required.
	
	In the sketch above, we have ignored some technical issues. For example, we need to be careful that $M$ has a bounded spectral norm, so that it induces a linear transformation which does not distort norms by too much (as all our arguments apply for input distributions supported on a bounded domain). A second issue is that if we apply a linearly-invariant algorithm on a dataset transformed by $M$, then the invariance is only with respect to the data, not necessarily with respect to new instances $\bx$ sampled from the same distribution (and this restriction is necessary for results such as \thmref{thm:affuninv} to hold without further assumptions). However, it can be shown that if the dataset is large enough, invariance will still occur with high probability over the sampling of $\bx$, which is sufficient for our purposes.

	\section{Natural Input Distributions}\label{sec:input}
	
	In this section, we consider the difficulty of gradient-based methods to learn certain target functions, even with respect to smooth, well-behaved distributions over $\reals^d$. Specifically, we will consider functions of the form $\bx\mapsto \psi(\inner{\bw^\star,\bx})$, where $\bw^\star$ is a vector of bounded norm and $\psi$ is a periodic function. Note that if $\psi$ is continuous and piecewise linear, then $\psi(\inner{\bw^\star,\bx})$ can be implemented by a depth-2 neural ReLU network on any bounded subset of the domain. More generally, any continuous periodic function can be approximated arbitrarily well by such networks. 
	
	Our formal results rely on Fourier analysis and are a bit technical. Hence, we precede them with an informal description, outlining the main ideas and techniques, and presenting a specific case study which may be of independent interest (Subsection \ref{subsec:informal}). The formal results are presented in Subsection \ref{subsec:formal}. 
	
	\subsection{Informal Description of Results and Techniques}\label{subsec:informal}
	
	Consider a target function of the form $\bx\mapsto \psi(\inner{\bw^\star,\bx})$, and any input distribution whose density function can be written as the square $\varphi^2$ of some function $\varphi$ (the reason for this will become apparent shortly). Suppose we attempt to learn this target function (with respect to the squared loss) using \emph{some} hypothesis class, which can be parameterized by a bounded-norm vector $\bw$ in some subset $\Wcal$ of an Euclidean space (not necessarily of the same dimensionality as $\bw^\star$), so each predictor in the class can be written as $\bx\mapsto f(\bw,\bx)$ for some fixed mapping $f$. Thus, our goal is essentially to solve the stochastic optimization problem
	\begin{equation}\label{eq:objfun}
	\min_{\bw:\bw\in \Wcal}~ \E_{\bx\sim\varphi^2}\left[\left(f(\bw,\bx)-\psi(\inner{\bw^\star,\bx})\right)^2\right].
	\end{equation}
	In this section, we  study the geometry of this objective function, and show that under mild conditions on $f$, and assuming the norm of $\bw^\star$ is reasonably large, the following holds:
	\begin{itemize}
		\item For any fixed $\bw$, the value of the objective function is almost independent of $\bw^\star$, in the sense that if we pick the direction of $\bw^\star$ uniformly at random, the value is extremely concentrated around a fixed value independent of $\bw^\star$ (e.g. exponentially small in $\norm{\bw^\star}^2$ for a Gaussian or a mixture of Gaussians). 
		\item Similarly, the gradient of the objective function with respect to $\bw$ is almost independent of $\bw^\star$, and is extremely concentrated around a fixed value (again, exponentially small in $\norm{\bw^\star}^2$ for, say, a mixture of Gaussians).
	\end{itemize}
	Therefore, assuming $\norm{\bw^\star}$ is reasonably large, any standard gradient-based method will follow a trajectory nearly independent of $\bw^\star$. In fact, in practice we do not even have access to exact gradients of \eqref{eq:objfun}, but only to noisy and biased versions of it (e.g. if we perform stochastic gradient descent, and certainly if we use finite-precision computations). In that case, the noise will completely obliterate the exponentially small signal about $\bw^\star$ in the gradients, and will make the trajectory essentially independent of $\bw^\star$. As a result, assuming $\psi$ and the distribution is such that the function $\psi(\inner{\bw^\star,\bx})$ is sensitive to the direction of $\bw^\star$, it follows that these methods will fail to optimize \eqref{eq:objfun} successfully. Finally, we note that in practice, it is common to solve not \eqref{eq:objfun} directly, but rather its empirical approximation with respect to some fixed finite training set. Still, by concentration of measure, this empirical objective would converge to the one in \eqref{eq:objfun} given enough data, so the same issues will occur.
	
	An important feature of our results is that they make virtually no 
	structural 
	assumptions on the predictors $\bx\mapsto f(\bw,\bx)$. In particular, they 
	can represent arbitrary classes of neural networks (as well as other 
	predictor classes). Thus, our results imply that target functions of the 
	form $\bx\mapsto 
	\psi(\inner{\bw^\star,\bx})$, where $\psi$ is periodic, would be difficult 
	to learn using gradient-based methods, even if we allow improper learning 
	and consider predictor classes of a different structure.
	
	\begin{figure}[t]
		\centering
		\includegraphics[trim=2cm 10.5cm 2cm 10.5cm, clip=true,scale=0.8]{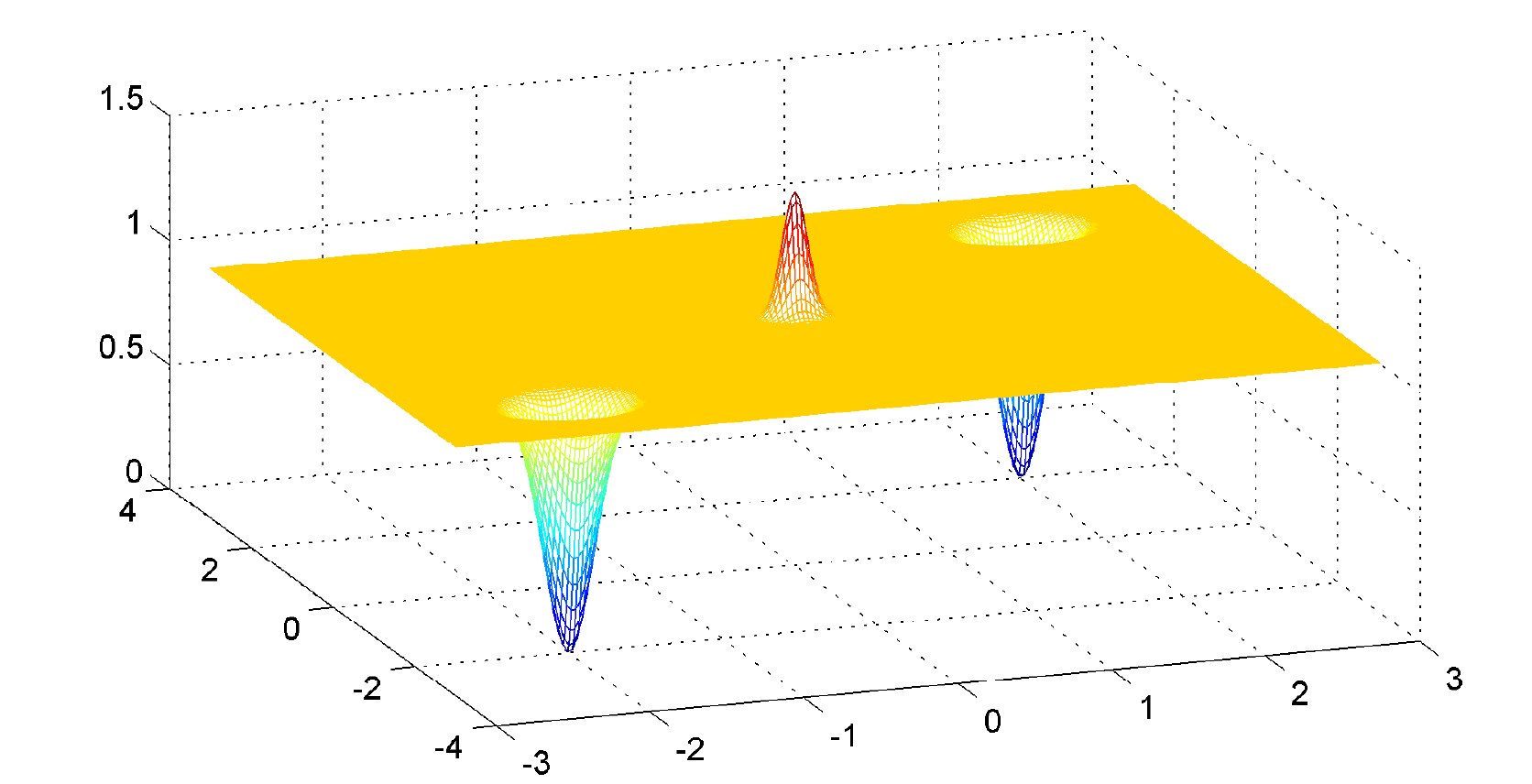}
		\caption{Graphical depiction of the objective function in \eqref{eq:objfuncos}, in 2 dimensions and where $\bw^\star=(2,2)$.}\label{fig:obj}
	\end{figure}
	
	To explain how such results are attained, let us study a concrete special 
	case (not necessarily in the context of neural networks). Consider the 
	target function $\bx\mapsto \cos(2\pi\inner{\bw^\star,\bx})$, and the 
	hypothesis class (parameterized by $\bw$) of functions $\bx\mapsto 
	\cos(2\pi\inner{\bw,\bx})$. Thus, \eqref{eq:objfun} takes the form
	\begin{equation}\label{eq:objfuncos}
	\min_{\bw} \E_{\bx\sim\varphi^2}\left[\left(\cos(2\pi\inner{\bw,\bx})-\cos(2\pi\inner{\bw^\star,\bx}\right)^2\right].
	\end{equation}
	Furthermore, suppose the input distribution  $\varphi^2$ is a standard 
	Gaussian on $\reals^d$. In two dimensions and for $\bw^\star=(2,2)$, the 
	objective function in \eqref{eq:objfun} turns out to have the form 
	illustrated in Figure \ref{fig:obj}. This objective function has only three 
	critical points: A global maximum at $\mathbf{0}$, and two global minima at 
	$\bw^\star$ and $-\bw^\star$. Nevertheless, it would be difficult to 
	optimize using gradient-based methods, since it is extremely \emph{flat} 
	everywhere except close to the critical points. As we will see shortly, the 
	same phenomenon occurs in higher dimensions. In high dimensions, if the 
	direction of $\bw^\star$ is chosen randomly, we will be overwhelmingly 
	likely to initialize far from the global minima, and hence will start in a 
	flat plateau in which most gradient-based methods will 
	stall\footnote{Although there are techniques to overcome flatness (e.g. by 
	normalizing the gradient \cite{nesterov1984minimization,hazan2015beyond}), 
	in our case the normalization factor will be huge and require extremely 
	precise gradient information, which as discussed earlier, is unrealistic 
	here.}.
	
	We now turn to explain why \eqref{eq:objfuncos} has the form shown in Figure \ref{fig:obj}. This will also help to illustrate our proof techniques, which apply much more generally. The main idea is to analyze the Fourier transform of \eqref{eq:objfuncos}. Letting $\cos_{\bw}$ denote the function $\bx\mapsto \cos(2\pi\inner{\bw,\bx})$, we can write \eqref{eq:objfuncos} as
	\[
	\int \left(\cos(2\pi\inner{\bw,\bx})-\cos(2\pi\inner{\bw^\star,\bx}\right)^2\varphi^2(\bx)d\bx ~=~ \norm{\cos_{\bw}\cdot\varphi-\cos_{\bw^\star}\cdot\varphi}^2,
	\]
	where $\norm{\cdot}$ is the standard norm over the space $L^2(\reals^d)$ of square integrable functions. By standard properties of the Fourier transform (as described in \secref{sec:prelim}), this squared norm of a function equals the squared norm of the function's Fourier transform, which equals in turn
	\[
	\norm{\widehat{\cos_{\bw}\cdot \varphi}-\widehat{\cos_{\bw}\cdot \varphi}}^2
	~=~ \norm{\widehat{\cos_{\bw}}*\hat{\varphi}-\widehat{\cos_{\bw^\star}}*\hat{\varphi}}^2.
	\]
	$\widehat{\cos_{\bw}}(\bxi)$ can be shown to equal $\frac{1}{2}\left(\delta(\bxi-\bw)+\delta(\bxi+\bw)\right)$, where $\delta(\cdot)$ is Dirac's delta function (a ``generalized'' function which satisfies $\delta(\bz)=0$ for all $\bz\neq \mathbf{0}$, and $\int \delta(\bz)d\bz=1$). Plugging this into the above and simplifying, we get
	\begin{align}
	&\frac{1}{4}\norm{\hat\varphi(\cdot-\bw)+\hat\varphi(\cdot+\bw)-\hat\varphi(\cdot-\bw^\star)-\hat\varphi(\cdot+\bw^\star)}^2\notag\\
	&~=~ \frac{1}{4}\int_{\bxi}\left|\hat\varphi(\bxi-\bw)+\hat\varphi(\bxi+\bw)-\hat\varphi(\bxi-\bw^\star)-\hat\varphi(\bxi+\bw^\star)\right|^2d\bxi.\label{eq:fourierobjbefore}
	\end{align}
	If $\varphi^2$ is a standard Gaussian, $\hat{\varphi}(\bxi)$ can be shown to equal the Gaussian-like function $(4\pi)^{d/2}a^{-\norm{\bxi}^2}$ where $a=\exp(4\pi^2)$. Plugging back, the expression above is proportional to 
	\begin{equation}\label{eq:fourierobj}
	\int_{\bxi} \left(\left(a^{-\norm{\bxi-\bw}^2}+a^{-\norm{\bxi+\bw}^2}\right)-\left(a^{-\norm{\bxi-\bw^\star}^2}+a^{-\norm{\bxi+\bw^\star}^2}\right)\right)^2d\bxi.
	\end{equation}
	The expression in each inner parenthesis can be viewed as a mixture of two Gaussian-like functions, with centers at $\bw,-\bw$ (or $\bw^\star,-\bw^\star$). Thus, if $\bw$ is far from $\bw^\star$, these two mixtures will have nearly disjoint support, and \eqref{eq:fourierobj} will have nearly the same value regardless of $\bw$ -- in other words, it is very flat. Since this equation is nothing more than a re-formulation of the original objective function in \eqref{eq:objfuncos} (up to a constant), we get a similar behavior for \eqref{eq:objfuncos} as well.
	
	This behavior extends, however, much more generally than the specific objective in \eqref{eq:objfuncos}. First of all, we can replace the standard Gaussian distribution $\varphi^2$ by any distribution such that $\hat{\varphi}$ has a localized support. This would still imply that \eqref{eq:fourierobjbefore} refers to the difference of two functions with nearly disjoint support, and the same flatness phenomenon will occur. Second, we can replace the $\cos$ function by any periodic function $\psi$. By properties of the Fourier transform of periodic functions, we still get localized functions in the Fourier domain (more precisely, the Fourier transform will be localized around integer multiples of $\bw$, up to scaling). Finally, instead of considering hypothesis classes of predictors $\bx\mapsto\psi(\inner{\bw,\bx})$ similar to the target function, we can consider quite arbitrary mappings $\bx\mapsto f(\bw,\bx)$. Even though this function may no longer be localized in the Fourier domain, it is enough that only the target function $\bx\mapsto \psi(\inner{\bw^\star,\bx})$ will be localized: That implies that regardless how $f$ looks like, under a random choice of $\bw^\star$, only a minuscule portion of the $L_2$ mass of $f$ overlaps with the target function, hence getting sufficient signal on $\bw^\star$ will be difficult. 
	
	As mentioned in the introduction, these techniques and observations have 
	some close resemblance to hardness results for learning parities over the 
	Boolean cube in the statistical queries learning model 
	\cite{blum1994weakly}. There as well, one considers a Fourier transform 
	(but on the Boolean cube rather than Euclidean space), and essentially show 
	that functions with a ``localized'' Fourier transform are difficult to 
	``detect'' using any fixed function. However, our results are different and 
	more general, in the sense that they apply to generic smooth distributions 
	over Euclidean space, and to a general class of periodic functions, rather 
	than just parities. On the flip side, our results are constrained to 
	methods which are based on gradients of the objective, whereas the 
	statistical queries framework is more general and considers algorithms 
	which are based on computing (approximate) expectations of arbitrary 
	functions of the data. Extending our results to this generality is an 
	interesting topic for future research.

	\subsection{Formal Results}\label{subsec:formal}
	
	We now turn to provide a more formal statement of our results. The distributions we will consider consist of arbitrary mixtures of densities, whose square roots have rapidly decaying tails in the Fourier domain. More precisely, we have the following definition:
		
		\begin{definition}
			Let $\epsilon(r)$ be some function from $[0,\infty)$ to $[0,1]$. A function $\varphi^2:\reals^d\rightarrow\reals$ is \emph{$\epsilon(r)$ Fourier-concentrated} if its square root $\varphi$ belongs to $L^2(\reals^d)$, and satisfies
			\[
			\norm{\hat{\varphi} \cdot \mathbf{1}_{\geq r}} ~\leq~ \norm{\hat{\varphi}}\epsilon(r),
			\]
			where $\mathbf{1}_{\geq r}$ is the indicator function of $\{\bx:\norm{\bx}\geq r\}$.
		\end{definition}
		
		A canonical example is Gaussian distributions: Given a (non-degenerate, 
		zero-mean) Gaussian density function $\varphi^2$ with covariance matrix 
		$\Sigma$, its square root $\varphi$ is proportional to a Gaussian with 
		covariance $2\Sigma$, and its Fourier transform $\hat{\varphi}$ is 
		well-known to be proportional to a Gaussian with covariance 
		$(2\Sigma)^{-1}$. By standard Gaussian concentration results, it 
		follows that $\varphi^2$ is Fourier-concentrated with 
		$\epsilon(r)=\exp(-\Omega(\lambda_{\min}r^2))$ where $\lambda_{\min}$ 
		is the minimal eigenvalue of $\Sigma$. A similar bound can be shown 
		when the Gaussian has some arbitrary mean. More generally, it is 
		well-known that smooth functions (differentiable to sufficiently high 
		order with integrable derivatives) have Fourier transforms with rapidly 
		decaying tails. For example, if we consider the broad class of 
		\emph{Schwartz} functions (characterized by having values and all 
		derivatives decaying faster than polynomially in $r$), then the 
		Fourier transform of any such function is also a Schwartz function, 
		which implies super-polynomial decay of $\epsilon(r)$ (see for instance 
		\cite{HNBook}, Chapter 11 and Proposition 11.25).

	We now formally state our main result for this section. We consider \emph{any} predictor of the form $\bx\mapsto f(\bw,\bx)$, where $f$ is some fixed function and $\bw$ is a parameter vector coming from some domain $\Wcal$, which we will assume w.l.o.g. to be a subset of some Euclidean space\footnote{More generally, our analysis is applicable to any separable Hilbert space.} (for example, $f$ can represent a network of a given architecture, with weights specified by $\bw$). When learning $f$ based on data coming from an underlying distribution, we are essentially attempting to solve the optimization problem
	\[
	\min_{\bw\in\Wcal} F_{\bw^\star}(\bw) := \E_{\bx\sim \varphi^2}\left[\left(f(\bw,\bx)-\psi(\inner{\bw^\star,\bx})\right)^2\right].
	\]
	Assume that $F$ is differentiable w.r.t. $\bw$, any gradient-based method to solve this problem proceeds by computing (or approximating) $\nabla F_{\bw^\star}(\bw)$ at various points $\bw$. However, the following theorem shows that at \emph{any} $\bw$, and \emph{regardless} of the type of predictor or network one is attempting to train, the gradient at $\bw$ is virtually independent of the underlying target function, and hence provides very little signal: 
	
	\begin{theorem}\label{thm:main}
		Suppose that
		\begin{itemize}
			\item $\psi:\reals\rightarrow[-1,+1]$ is a periodic function of period $1$, which has bounded variation on every finite interval. \item $\varphi^2$ is a density function on $\reals^d$, which can be written as a (possibly infinite) mixture $\varphi^2=\sum_i \alpha_i \varphi_i^2$, where each $\varphi_i$ is an $\epsilon(r)$ Fourier-concentrated density function.
			\item At some fixed $\bw$, 
			$\E_{\bx\sim\varphi^2}\norm{\frac{\partial}{\partial 
			\bw}f(\bw,\bx)}^2\leq G_{\bw}$ for some $G_{\bw}$.
		\end{itemize}	
		Then for some universal positive constants $c_1,c_2,c_3$, if $d\geq c_1$, and $\bw^\star\in\reals^d$ is a vector of norm $2r$ chosen uniformly at random, then
		\[
		\var_{\bw^\star}\left(\nabla F_{\bw^\star}(\bw)\right)~:=~ 
		\E_{\bw^\star}\left\|\nabla F_{\bw^\star}(\bw)-\E_{\bw^\star}[\nabla 
		F_{\bw^\star}(\bw)]\right\|^2~\leq~ c_2 
		G_{\bw}\left(\exp(-c_3d)+\sum_{n=1}^{\infty}\epsilon(nr)\right).
		\]
	\end{theorem}
	We note that bounded variation is weaker than, say, Lipschitz continuity. 
	Assuming $\epsilon(r)$ decays rapidly with $r$ -- say, exponentially in 
	$r^2$ as is the case for a Gaussian mixture -- we get that the bound in the 
	theorem is on the order of $\exp(-\Omega(\min\{d,r^2\}))$.
	
	Overall, the theorem implies that if $r,d$ are moderately large, the 
	gradient of $F_{\bw^\star}$ at any point $\bw$ is extremely concentrated 
	around a fixed value, independent of $\bw^\star$. This implies that 
	gradient-based methods, which attempt to optimize $F_{\bw^\star}$ via 
	gradient information, are unlikely to succeed. One way to formalize this is 
	to consider any iterative algorithm (possibly randomized), which relies on 
	an 
	\emph{$\varepsilon$-approximate} gradient oracle to optimize 
	$F_{\bw^\star}$: 
	At every iteration $t$, the algorithm chooses a point $\bw_t\in\Wcal$, and 
	receives 
	a vector $\bg_t$ such that $|\nabla F_{\bw^\star}-\bg_t|\leq \varepsilon$. 
	In 
	our case, we will be interested in $\varepsilon$ such that $\varepsilon^3$ 
	is on the 
	order of the bound in \thmref{thm:main}. Since the bound is extremely 
	small for moderate $d,r$ (say, smaller than machine precision), this is a 
	realistic model of gradient-based methods on finite-precision machines, 
	even if one attempts 
	to compute the gradients accurately. The following theorem implies that 
	if the number of iterations is not extremely large (on the order of 
	$1/\varepsilon$, e.g. $\exp(\Omega(d,r^2))$ iterations for Gaussian 
	mixtures), then 
	with high probability, a gradient-based method will return the same 
	predictor independent of $\bw^\star$. However, since the objective function
	$F_{\bw^\star}$ is highly sensitive to the choice of $\bw^\star$, this 
	means that no such  gradient-based method can train a reasonable predictor.
	\begin{theorem}\label{thm:trajectory}
	Assume the conditions of \thmref{thm:main}, and let $\varepsilon = 
	\sqrt[3]{c_2 
	(\sup_{\bw\in\Wcal}G_{\bw})\left(\exp(-c_3d)
	+\sum_{n=1}^{\infty}\epsilon(nr)\right)}$
	be the cube root of the bound specified there (uniformly over
	all $\bw\in\Wcal$). Then 
	for any algorithm as above and any $p\in (0,1)$, conditioned on an 
	event which holds with probability $1-p$ over the choice of 
	$\bw^\star$, its output after at most
	$p/\varepsilon$ iterations will be independent of 
	$\bw^\star$. 
	\end{theorem}

	\section{Proofs}\label{sec:proofs}

	\subsection{Proof of \thmref{thm:affuninv}}
	
	Let $P_M$ denote the whitening matrix employed if we transform the instances $X$ by some invertible $d\times d$ matrix $M$ (that is, $X$ becomes $MX$), and $P$ the whitening matrix employed for the original data.
	
	Using the same notation as in the theorem, it is easily verified that $P X = V^\top$, and $P_{M} MX = V_{M}^\top$, where $U_{M} D_{M} V_{M}^\top$ is an SVD decomposition of the matrix $MX$. Since both $V^\top$ and $V_{M}^\top$ are $\text{Rank}(X)\times m$ matrices with rows consisting of orthonormal vectors, they are related by an orthogonal transformation (i.e. there is an orthogonal matrix $R_M$ such that $R_M V^\top = V_{M}^\top$). Therefore, $R_M P X = P_M M X$. Since the data is fed to an orthogonally-invariant algorithm, its output $W_M$ satisfies $W_M^\top P_{M} M X = W^\top PX$. This in turn implies $W_M^\top R_M P X = W^\top P X$, and hence $W_M^\top R_M V^\top = W^\top V^\top$. Multiplying both sides on the right by $V$ and taking a transpose, we get that $R_M^\top W_M = W$, and hence $W_M = R_M W$. In words, $W$ and $W_M$ are the same up to an orthogonal transformation $R_M$ depending on $M$. Therefore,
	\[
	(P_M^\top W_M)^\top MX = W_M^\top P_M MX = W^\top R_M^\top R_M P X = W^\top P X = (P^\top W)^\top X,
	\]
	so we see that the returned predictor makes the same predictions over the dataset, independent of the transformation matrix $M$.

	\subsection{Proof of \thmref{thm:maintarget}}
	
	We start with the following auxiliary theorem, which reduces the hardness result of \cite{daniely2016complexity} to one about neural networks of the type we discuss here:
	
	\begin{theorem}\label{thm:reduc}
		Under the assumption stated in \cite{daniely2016complexity}, the following holds for any $n_d=\omega(1)$ (as $d\rightarrow \infty$):
		
		There is no algorithm running in time $\text{poly}(d,1/\epsilon)$, which for any distribution $\Dcal$ on $\{0,1\}^d$, and any $h(W^\top \bx)=\sigma(\sum_{i=1}^{n_d}[\inner{\bw_i,\bx}]_+)$ (where $W=[\bw_1,\bw_2,\ldots,\bw_{n_d}]$ and $\max_i \norm{\bw_i}\leq \Ocal(d)$), given only access to samples $(\bx,h(W^\top\bx))$ where $\bx\sim\Dcal$, returns with high probability a function $f$ such that
		\[
		\E_{\bx\sim \Dcal}\left[\left(f(\bx)-h(W^\top\bx)\right)^2\right]\leq \epsilon.
		\]
	\end{theorem}
	
	\begin{proof}
		Suppose by contradiction that there exists an algorithm $\Acal$ which for any distribution and $h_W$ as described in the theorem, returns a function $f$ such that $\E_{\bx\sim\Dcal}\left[\left(f(\bx)-h(W^\top\bx)\right)^2\right]\leq \epsilon$ with high probability.
		
		In particular, let us focus on distributions $\Dcal$ supported on $\{0,1\}^{d-1}\times \{1\}$. For these distributions, we argue that any intersection of halfspaces on $\reals^{d-1}$ specified by  $\bw_1,\ldots,\bw_{n_d}\in \reals^{d-1}$ with integer coordinates, and integer $b_1,\ldots,b_n$, can be specified as $\bx\mapsto 1-h(W^\top\bx)$ for some function $h$ as described in the theorem statement. To see this, note that for any $\bw_i,b_i$ and $\bx=(\bx',1)$ in the support of $\Dcal$, $\inner{(-\bw_i,b_i),\bx} = -\inner{\bw_i,\bx'}+b_i$ is an integer, hence 
		\begin{align*}
		\sigma\left(\sum_{i=1}^{n_d}[\inner{(-\bw_i,b_i),\bx}]_+\right) &=
		\sigma\left(\sum_{i=1}^{n_d}[-\inner{\bw_i,\bx'}+b_i]_+\right)~=~
		\bigvee_{i=1}^{n_d} \left(-\inner{\bw_i,\bx'}+b_i>0\right)\\
		&=\bigvee_{i=1}^{n_d} \left(\inner{\bw_i,\bx'}< b_i\right) ~=~
		\neg\left(\bigwedge_{i=1}^{n_d}\left(\inner{\bw_i,\bx'}\geq b_i\right)\right).
		\end{align*}
		Therefore, for any distribution over examples labelled by an intersection of halfspaces $\bx\mapsto 1-h(W^\top \bx)$ (with integer-valued coordinates and bounded norms), by feeding $\Acal$ with $\{\bx_i,1-y_i\}_{i=1}^{m}$, the algorithm returns a function $f$, such that with high probability, $\E_{\bx}\left(f(\bx)-h(W^\top \bx)\right)^2\leq \epsilon$, and therefore
		\[
		\E_{\bx}\left((1-f(\bx))-(1-h(W^\top\bx))\right)^2  \leq \epsilon.
		\]
		In particular, if we consider the Boolean function $\tilde{f}(\bx) = 1-\text{rnd}(f(\bx))$, where $\text{rnd}(z)=0$ if $z\leq 1/2$ and $\text{rnd}(z)=1$ if $z>1/2$, we argue that $\Pr_{\bx}(\tilde{f}(\bx)\neq 1-h(W^\top\bx))\leq 8\epsilon$. Since $\epsilon$ is arbitrary, and $1-h(W^\top\bx)$ specifies an intersection of halfspaces, this would contradict the hardness result of \cite{daniely2016complexity}, and therefore prove the theorem. This argument follows from the following chain of inequalities, where $\mathbf{1}$ denotes the indicator function:
		\begin{align*}
		\Pr\left(\tilde{f}(\bx))\neq g(\bx)\right) &= \Pr(f(\bx)> 1/2~\wedge~g(\bx)=1)+\Pr(f(\bx)\leq 1/2~\wedge~g(\bx)=0)\\
		&=\E\left[\mathbf{1}\left(f(\bx)> 1/2~\wedge~g(\bx)=1\right)\right]+\E\left[\mathbf{1}\left(f(\bx)\leq 1/2~\wedge~g(\bx)=0\right)\right]\\
		&\leq \E\left[4\left((1-f(\bx))-g(\bx)\right)^2\right]+\E\left[4\left((1-f(\bx))-g(\bx)\right)^2\right]\\
		&\leq 8\cdot\E\left[\left((1-f(\bx))-g(\bx)\right)^2\right] ~\leq~ 8\epsilon.
		\end{align*}
	\end{proof}
	
	\begin{proposition}\label{prop:reduc}
		\thmref{thm:reduc} holds even if we restrict $\bw_1,\ldots,\bw_{n_d}$ to be linearly independent, with $s_{\min}(W)\geq 1$.
	\end{proposition}
	\begin{proof}
		Suppose by contradiction that there exists an algorithm $\Acal$ which succeeds for any $W$ as stated above. We will describe how to use $\Acal$ to get an algorithm which succeeds for any $W$ as described in \thmref{thm:reduc}, hence reaching a contradiction.
		
		Specifically, suppose we have access to samples $(\bx,h(W^\top \bx))$, where $\bx$ is supported on $\{0,1\}^d$, and where $W$ is any matrix as described in \thmref{thm:reduc}. We do the following: We map every $\bx$ to $\tilde{\bx} \in \{0,1\}^{d+n_d}$ by $\tilde{\bx}=(\bx,0,\ldots,0)$, run $\Acal$ on the transformed samples $(\tilde{\bx},h(W^\top\bx))$ to get some predictor $\tilde{f}:\{0,1\}^{d+n_d}\mapsto \reals$, and return the predictor $f(\bx) = \tilde{f}((\bx,0,\ldots,0))$. 
		
		To see why this reduction works, we note that the mapping $\bx\mapsto\tilde{\bx}$ we have defined, where $\bx$ is distributed according to $\bx$, induces a distribution $\tilde{\Dcal}$ on  $\{0,1\}^{d+n_d}$. Let $\tilde{W}$ be the $(d+n_d)\times n_d$ matrix $[W;I_{n_d}]$ (that is, we add another $n_d\times n_d$ unit matrix below $W$). We have $\tilde{W}^\top\tilde{W} = W^\top W + I_{n_d}$, so the minimal eigenvalue of $\tilde{W}^\top\tilde{W}$ is at least $1$, hence $s_{\min}(\tilde{W})\geq 1$, so $\tilde{W}$ satisfies the conditions in the proposition. Moreover, the norm of each column of $\tilde{W}$ is larger than the norm of the corresponding column in $W$ by at most $1$, so the norm constraint in \thmref{thm:reduc} still holds. Finally, $\tilde{W}\tilde{\bx}=W\bx$ for all $\bx$, and therefore $h_{\tilde{W}}(\tilde{\bx}) = h_{W}(\bx)$. Thus, the distribution of $(\tilde{\bx},h(W^\top\bx))=(\tilde{\bx},h(\tilde{W}^\top\bx))$ (which is used to feed the algorithm $\Acal$) is a valid distribution corresponding to the conditions of the proposition and \thmref{thm:reduc}  (only in dimension $d+n_d\leq 2d$ instead of $d$), so $\Acal$ returns with high probability a predictor $\tilde{f}$ such that 
		\[
		\E_{\tilde{\bx}\sim \tilde{\Dcal}}\left[\left(\tilde{f}(\tilde{\bx})-h(\tilde{W}^\top\tilde{\bx})\right)^2\right]\leq \epsilon.
		\]	
		However, $\tilde{f}(\tilde{\bx})=f(\bx)$,  $h(\tilde{W}^\top\tilde{\bx}) = h(W^\top\bx)$, so the returned predictor $f$ satisfies
		\[ \E_{\bx\sim\Dcal}\left[\left(f(\bx)-h(W^\top\bx)\right)^2\right]\leq \epsilon.
		\]
		This contradicts \thmref{thm:reduc}, which states that no efficient algorithm can return such a predictor for \emph{any} sufficiently large dimension $d$ and norm bound $\Ocal(d)$. 
	\end{proof}
	
	In the definitions of orthogonal invariance and linear invariance, we only required the invariance to hold with respect to instances $\bx_i$ in the dataset. A stronger condition is that the invariance is satisfied for any $\bx\in \reals^d$. However, the following lemma shows that invariance w.r.t. a dataset sampled i.i.d. from some distribution is sufficient to imply invariance w.r.t. ``nearly all'' $\bx$ (under the same distribution):
	
	\begin{lemma}\label{lem:allinv}
		Suppose the dataset $\{\bx_i,y_i\}_{i=1}^{m}$ is sampled i.i.d. from some distribution (where $\bx_i\in \reals^d$), then the following holds with probability at least $1-\delta$ for any $\delta\in (0,1)$: For any invertible $M$ and linearly-invariant algorithm (or orthogonal $M$ and orthogonally-invariant algorithm), conditioned on the algorithm's internal randomness, the returned matrices $W$ and $W_M$ (with respect to the original data and the data transformed by $M$ respectively) satisfy
		\[
		{\Pr}_{\bx}(W_M^\top M\bx \neq W^\top \bx)\leq \frac{d}{\delta(m+1)}.
		\]
	\end{lemma}
	\begin{proof}
		It is enough to prove that with probability at least $1-\delta$ over the sampling of $\bx_1,\ldots,\bx_m$,
		\begin{equation}\label{eq:spanshow} 
		{\Pr}_{\bx}(\bx\notin\text{span}(\bx_1,\ldots,\bx_m))\leq \frac{d}{\delta(m+1)}.
		\end{equation}
		This is because the event $W_M^\top M\bx_i = W^\top\bx_i$ for all $i$ means that $W_M^\top M\bx= W^\top\bx$ for any $\bx$ in the span of $\bx_1,\ldots,\bx_m$. 
		
		Let $\bx_1,\ldots,\bx_{m+1}$ be sampled i.i.d. according to $\Dcal$. Considering probabilities over this sample, we have
		\begin{equation}\label{eq:sumpr}
		\sum_{j=1}^{m+1}\Pr\left(\bx_{j}\notin \text{span}(\bx_1,\ldots,\bx_{j-1})\right) ~=~
		\E\left[\sum_{j=1}^{m+1}\mathbf{1}\left(\bx_{j}\notin \text{span}(\bx_1,\ldots,\bx_{j-1})\right)\right] \leq d,
		\end{equation}
		where the latter inequality is because each $\bx_j$ is a $d$-dimensional vector, hence the number of times we can get a vector not in the span of the previous ones is at most $d$. Moreover, since the vectors are sampled i.i.d, we have
		\[
		\Pr(\bx_{j+1}\notin \text{span}(\bx_1,\ldots,\bx_{j}))~\leq~ \Pr(\bx_{j+1}\notin \text{span}(\bx_1,\ldots,\bx_{j-1}))~=~\Pr(\bx_{j}\notin \text{span}(\bx_1,\ldots,\bx_{j-1})),
		\]
		so the probabilities in \eqref{eq:sumpr} monotonically decrease with $j$. Thus, \eqref{eq:sumpr} implies
		\[
		(m+1)\Pr\left(\bx_{m+1}\notin \text{span}(\bx_1,\ldots,\bx_{m})\right) \leq d ~~\Rightarrow~~ \Pr\left(\bx_{m+1}\notin \text{span}(\bx_1,\ldots,\bx_{m})\right)\leq \frac{d}{m+1}. 
		\]
		This is equivalent to
		\[
		\E_{\bx_1,\ldots,\bx_m\sim \Dcal}\left[{\Pr}_{\bx_{m+1}}\left(\bx_{m+1}\notin\text{span}(\bx_1,\ldots,\bx_m)\middle| \bx_1,\ldots,\bx_m\right)\right]\leq \frac{d}{m+1},
		\]
		so by Markov's inequality, with probability at least $1-\delta$ over the sampling of $\bx_1,\ldots,\bx_m$,
		\[
		{\Pr}_{\bx_{m+1}}\left(\bx_{m+1}\notin\text{span}(\bx_1,\ldots,\bx_m)\right) \leq \frac{d}{\delta(m+1)}.
		\]
		Since $\bx_{m+1}$ is sampled independently, \eqref{eq:spanshow} and hence the lemma follows.
	\end{proof}
	
	With these results in hand, we can finally turn to prove \thmref{thm:maintarget}. Suppose by contradiction that there exists an efficient linearly-invariant algorithm $\Acal$, which for any distribution $\Dcal$ supported on vectors of norm at most $\Ocal\left(d\sqrt{2 d n_d}\right)/\min\{1,s_{\min}(W_\star)\}$, returns w.h.p. a predictor $\bx\mapsto f(\tilde{W}^\top \bx)$ such that
	\[
	\E_{\bx\sim \Dcal^{\star}}\left[\left(f(\tilde{W}^\top\bx)-h(W_\star^\top\bx)\right)^2\right] \leq \epsilon.
	\]
	We will show that the very same algorithm, if given  $\text{poly}(d,1/\epsilon)$ samples, can successfully learn w.r.t. \emph{any} $d\times n_d$ matrix $W$ and any distribution $\Dcal$ satisfying Proposition \ref{prop:reduc} and \thmref{thm:reduc}, contradicting those results.
	
	Indeed, let $W$ and $\Dcal$ be an arbitrary matrix and distribution as above. We first argue that there exists a $d\times d$ invertible matrix $M$ such that
	\begin{equation}\label{eq:mmm}
	W = M^\top W_\star~~~,~~~ \norm{M} \leq \frac{\Ocal(d\sqrt{2n_d})}{\min\{1,s_{\min}(W_\star)\}}.
	\end{equation}
	To see this, note that $W$ and $W_\star$ are of the same size and our conditions imply that both of them have full column rank. Thus, we can simply augment them to invertible $d\times d$ matrices $[W~\hat{W}]$ and $[W_\star~\hat{W}_{\star}]$, where the columns of $\hat{W}$ (respectively $\hat{W}_{\star}$) are an orthonormal basis for the subspace orthogonal to the column space of $W$ (respectively $W_\star$), and choosing $M^\top = [W~\hat{W}] [W_\star~\hat{W}_{\star}]^{-1}$. Thus, $\norm{M} \leq \norm{[W~\hat{W}]}\cdot\norm{[W_\star~\hat{W}_{\star}]^{-1}}$. The spectral norm of $[W~\hat{W}]$ can be bounded by the Frobenius norm, which by the assumption on $W$ from \thmref{thm:reduc} and the fact that $\hat{W}$ consist of an orthonormal basis, is at most $\sqrt{\Ocal(d)^2\cdot n_d + 1\cdot (d-n_d)} = \Ocal(\sqrt{d^2 n_d}) = \Ocal(d\sqrt{n_d})$. The spectral norm of $[W_\star~\hat{W}_{\star}]^{-1}$ can be bounded by $1/s_{\min}([W_\star~\hat{W}_{\star}])$, where $s_{\min}([W_\star~\hat{W}_{\star}])$ equals the square root of the smallest eigenvalue of $[W_\star~\hat{W}_{\star}]^\top [W_\star~\hat{W}_{\star}]$, which can be easily verified to be $\min\{1,s_{\min}(W_\star)\}$. 
	
	Now, consider the following thought experiment: Suppose we would run the algorithm $\Acal$ with samples $(M\bx_i,h(W_\star^\top (M\bx_i))$, $i=1,2,\ldots,m$, where $\bx_i$ is sampled from $\Dcal$ (which by the assumptions of \thmref{thm:reduc}, is supported on vectors of norm at most $\sqrt{d}$). By \eqref{eq:mmm}, $M\bx_i$ is supported on the set of vectors of norm at most $\norm{M}\norm{\bx_i} \leq \norm{M}\sqrt{d}\leq \Ocal(d\sqrt{d n_d})/\min\{1,s_{\min}(W_\star)\}$, and the outputs correspond to the network specified by $W_\star$. Therefore, by assumption, the algorithm $\mathcal{A}$ would return w.h.p. a matrix $\tilde{W}_M$ such that
	\[
	\E_{\bx\sim \Dcal}\left[\left(f(\tilde{W}_M^\top (M\bx))-h(W_\star^\top(M\bx))\right)^2\right]\leq \epsilon.
	\]
	By \eqref{eq:mmm}, $W_\star^\top M = (M^\top W_\star)^\top = W^\top$, so this is equivalent to
	\begin{equation}\label{eq:Mtrans}
	\E_{\bx\sim \Dcal}\left[\left(f((\tilde{W}_M^\top M\bx)-h(W^\top\bx))\right)^2\right]\leq \epsilon.
	\end{equation}
	Let $\tilde{W}_I$ be the matrix returned by $\Acal$ if we had fed it with the samples $\{(\bx_i,h(W^\top \bx_i))\}_{i=1}^{m}$ (or equivalently, $\{(\bx_i,h(W_\star^\top (M\bx_i)))\}_{i=1}^{m}$)\footnote{Note that if the algorithm is stochastic, both $\tilde{W}_M$ and $\tilde{W}_I$ are not fixed given the data, but also depend on the algorithm's internal randomness. However, the proof will still follow by conditioning on any possible realization of this randomness.}. Let $E_{\bx}$ be the event (conditioned on the samples used by the algorithm) that a freshly sampled $\bx\sim \Dcal$ satisfies $\tilde{W}_M^\top M\bx = \tilde{W}_I^\top \bx$. By \lemref{lem:allinv}, w.h.p. over the samples fed to the algorithm, $\Pr_{\bx}(E_{\bx}) \geq 1-\Ocal(d/m)$. Therefore, w.h.p. over the samples $\bx_1,\ldots,\bx_m$,
	\begin{align*}
	\E_{\bx\sim\Dcal}&\left[\left(f(\tilde{W}^\top_I\bx)-h(W^\top \bx)\right)^2\right]\\
	&= \Pr_{\bx}(E_{\bx})\cdot	\E_{\bx}\left[\left(f(\tilde{W}^\top_I\bx)-h(W^\top \bx)\right)^2\middle|E_{\bx}\right]
	+\Pr_{\bx}(\neg E_{\bx})\cdot	\E_{\bx}\left[\left(f(\tilde{W}_I^\top\bx)-h(W^\top \bx)\right)^2\middle|\neg E_{\bx}\right]\\
	&\leq \Pr_{\bx}(E_{\bx})\cdot	\E_{\bx}\left[\left(f(\tilde{W}_M^\top M\bx)-h(W^\top \bx)\right)^2\middle|E_{\bx}\right]
	+\Pr_{\bx}(\neg E_{\bx})\cdot 1\\
	&= \E_{\bx}\left[\left(f(\tilde{W}_M^\top M\bx)-h(W^\top \bx)\right)^2\mathbf{1}(E_\bx)\right]+\Ocal\left(\frac{d}{m}\right)\\
	&\leq \E_{\bx}\left[\left(f(\tilde{W}_M^\top M\bx)-h(W^\top \bx)\right)^2\right]+\Ocal\left(\frac{d}{m}\right)~~\leq~~ \epsilon+\Ocal\left(\frac{d}{m}\right),
	\end{align*} 
	where we used the fact that $h$ maps $\bx$ to $[0,1]$, a union bound and \eqref{eq:Mtrans}. Now, recall that $\tilde{W}_I$ refers to the output of the algorithm, given samples $\{(\bx_i,h(W^\top \bx_i))\}_{i=1}^{m}$ where $m=\text{poly}(d,1/\epsilon)$. Thus, we have shown that w.h.p., as long as the algorithm is fed with $m\geq d/\epsilon$ samples\footnote{Even if the algorithm does not require that many samples, we can still artificially add more samples -- these are merely used to ensure that its linear invariance is with respect to a sufficiently large dataset.}, the algorithm returns $\tilde{W}_I$ which satisfies
	\[
	\E_{\bx\sim\Dcal}\left[\left(f(\tilde{W}_I^\top\bx)-h(W^\top \bx)\right)^2\right] ~\leq~ \Ocal(\epsilon).
	\]
	This means that the algorithm succesfully learns the hypothesis $\bx\mapsto h(W^\top \bx)$ with respect to the distribution $\Dcal$. Since $\epsilon$ is arbitrarily small and $W,\Dcal$ were chosen arbitrarily, the result follows.

	\subsection{Proof of \thmref{thm:main}}\label{subsec:proofthmmain}
	
	We note that for any function $q$, $\E_{\bx\sim \varphi^2}[q(\bx)] = \sum_{i}\alpha_i\cdot \E_{\bx\sim \varphi_i^2}[q(\bx)]$. Thus, it is enough to prove the bound in the theorem when $\varphi^2$ consists of a single element whose square root is Fourier-concentrated. For a mixture $\varphi^2 = \sum_i \alpha_i \varphi_i^2$, the result follows by applying the bound for each $\varphi_i^2$ individually, and using Jensen's inequality.
	
	To simplify notation a bit, we let $\psi_{\bw}(\cdot)$ stand for the function $\psi(\inner{\bw,\cdot})$, and let $h(\cdot-\bv)$ (where $\bv$ is some vector and $h$ is a function on $\reals^d$) stand for the function $\bx\mapsto h(\bx-\bv)$. Also, we will use several times the fact that for any two $L^2(\reals^d)$ functions $h_1,h_2$,
	\[
	\inner{h_1(\cdot-\bv),h_2(\cdot-\bv)} ~=~ \int h_1(\bx-\bv)\overline{h_2(\bx-\bv)}d\bx ~=~ \int h_1(\bx)\overline{h_2(\bx)}d\bx = \inner{h_1,h_2}.
	\]
	In other words, inner products (and hence also norms) in $L^2(\reals^d)$ are invariant to a shift in coordinates.
	
	The proof is a combination of a few lemmas, presented below.
	
	\begin{lemma}\label{lem:psivarph}
		For any $\bw$, it holds that $\psi_{\bw}\varphi\in L^2(\reals^d)$, and satisfies
		\[
		\widehat{\psi_{\bw}\varphi}(\bx) ~=~ \sum_{z\in\mathbb{Z}} a_z\cdot \hat{\varphi}(\bx-z\bw)
		\]
		for any $\bx$, where $\mathbb{Z}$ is the set of integers and  $a_z$ are complex-valued coefficients (corresponding to the Fourier series expansion of $\psi$, hence depending only on $\psi$) which satisfy $\sum_{z\in\mathbb{Z}}|a_z|^2 \leq 1$.
	\end{lemma}
	\begin{proof}
			First, we note that $\psi_{\bw}\varphi\in L^2(\reals^d)$, since both $\psi_{\bw}$ and $\varphi$ are locally integrable by the theorem's conditions, and satisfy
			\[
			\norm{\psi_{\bw}\varphi}^2 = \int \psi^2_{\bw}(\bx)\varphi^2(\bx)d\bx = \int \psi^2(\inner{\bw,\bx})\varphi^2(\bx)d\bx \leq \int \varphi^2(\bx)d\bx = 1 < \infty.
			\]
			As a result, $\widehat{\psi_{\bw}\varphi}$ exists as a function in $L^2(\reals^d)$.
			Since $\psi$ is a bounded variation function, it is equal everywhere to its Fourier series expansion:
			\[
			\psi(x) = \sum_{z\in \mathbb{Z}} a_z \exp\left(2\pi izx\right),
			\]
			where $i$ is the imaginary unit (note that since $\psi$ is real-valued, the imaginary components eventually cancel out, but it will be more convenient for us to represent the Fourier series in this compact form). By Parseval's identity, $\sum_{z} |a_z|^2 =\int_{-1/2}^{1/2}\psi^2(x)dx$, which is at most $1$ (since $\psi(x)\in [-1,+1]$).
			
			Based on this equation, we have
			\[
			\psi_{\bw}(\bx) = \psi(\inner{\bw,\bx}) = \sum_{z\in \mathbb{Z}} a_z \exp\left(2\pi iz\inner{\bw,\bx}\right).
			\]
			We now wish to compute the Fourier transform of the above\footnote{Strictly speaking, this function does not have a Fourier transform in the sense of \eqref{eq:fourierdef}, since the associated integrals do not converge. However, the function still has a well-defined Fourier transform in the more general sense of a generalized function or distribution (see e.g. \cite{HNBook} for a survey). In the derivation below, we will simply rely on some standard formulas from the Fourier analysis literature, and refer to \cite{HNBook} for their formal justifications.}. First, we note that the Fourier transform of $\exp(2\pi i \inner{\bv,\cdot})$ is given by $\delta(\cdot - \bv)$, where $\delta$ is the Dirac delta function (a so-called generalized function which satisfies $\delta(\bx)=0$ for all $\bx\neq \mathbf{0}$, and $\int \delta(\bx)d\bx=1$). Based on this and the linearity of the Fourier transform, we have that
			\[
			\hat{\psi}_{\bw}(\bx)~=~ \sum_{z\in \mathbb{Z}} a_z\cdot \delta\left(\bx-z\bw\right),
			\]
			and therefore, by the convolution property of the Fourier transform, we have 
			\begin{align}
			\widehat{\psi_{\bw}\varphi}(\bx) &= \left(\hat{\psi}_{\bw}*\hat{\varphi}\right)(\bx) = \int \hat{\psi}_{\bw}(\bz)\cdot \hat{\varphi}(\bx-\bz)d\bz\notag\\
			&= \sum_{z\in\mathbb{Z}} a_z\cdot \int \delta(\bz-z\bw)\cdot \hat{\varphi}(\bx-\bz)d\bz\notag\\
			&= \sum_{z\in\mathbb{Z}} a_z\cdot \hat{\varphi}(\bx-z\bw)\notag
			\end{align}
			as required.
	\end{proof}
	
	\begin{lemma}\label{lem:compind}
		For any distinct integers $z_1\neq z_2$ and any $\bw$ such that $\norm{\bw}= 2r$, it holds that
		\[
		\inner{~\left|\hat{\varphi}(\cdot-z_1\bw)\right|~,~\left|\hat{\varphi}(\cdot-z_2\bw)\right|~} ~\leq~ 2\cdot\epsilon(|z_1-z_2|r).
		\]
	\end{lemma}
	\begin{proof}
		Let $\Delta=|z_2-z_1|r$, and $\bv=(z_2-z_1)\bw$, so $\bv$ is a vector of norm $2\Delta$. 
		Since the inner product is invariant to shifting the coordinates, we can assume without loss of generality that $z_1=0$, and our goal is to bound $\inner{|\hat{\varphi}|,|\hat{\varphi}(\cdot-\bv)|}$.
		
		Using the convention that $\mathbf{1}_{\leq \Delta}$ is the indicator of $\{\bx:\norm{\bx}\leq \Delta\}$, and $\mathbf{1}_{>\Delta}$ is the indicator for its complement, we have
		\begin{align*}
			\inner{\left|\hat{\varphi}\right|,\left|\hat{\varphi}(\cdot-\bv)\right|} &=
			\inner{\left|\hat{\varphi}\right|~,~\left|\hat{\varphi}(\cdot-\bv)\right|\mathbf{1}_{\leq \Delta}} + \inner{\left|\hat{\varphi}\right|~,~\left|\hat{\varphi}(\cdot-\bv)\right|\mathbf{1}_{> \Delta}}\\			
			&= \inner{\left|\hat{\varphi}\right|~,~\left|\hat{\varphi}(\cdot-\bv)\right|\mathbf{1}_{\leq \Delta}} + \inner{\left|\hat{\varphi}\right|\mathbf{1}_{> \Delta}~,~\left|\hat{\varphi}(\cdot-\bv)\right|}\\
			&\leq \norm{\hat{\varphi}} \norm{\hat{\varphi}(\cdot-\bv)\mathbf{1}_{\leq \Delta}} ~+~\norm{\hat{\varphi}\mathbf{1}_{> \Delta}} \norm{\hat{\varphi}(\cdot-\bv)} ,
		\end{align*}
		where in the last step we used Cauchy-Schwartz. Using the fact that norms and inner products are invariant to coordinate shifting, the above is at most
		\[
		\norm{\hat{\varphi}} \norm{\hat{\varphi}\mathbf{1}_{\leq \Delta}(\cdot+\bv)}+
		\norm{\hat{\varphi}\mathbf{1}_{> \Delta}} \norm{\hat{\varphi}}~=~
		\norm{\hat{\varphi}}\left(\sqrt{\int |\hat{\varphi}(\bx)|^2\mathbf{1}_{\norm{\bx+\bv}\leq \Delta}d\bx}
		+\sqrt{\int |\hat{\varphi}(\bx)|^2\mathbf{1}_{>\Delta}(\bx)d\bx}\right).
		\]
		By the triangle inequality and the assumption $\norm{\bv}= 2\Delta$, the event $\norm{\bx+\bv}\leq \Delta$ implies $\norm{\bx}\geq \Delta$. Therefore, the above can be upper bounded by
		\[
		2\norm{\hat{\varphi}}\sqrt{\int |\hat{\varphi}(\bx)|^2\mathbf{1}_{\geq \Delta}(\bx)d\bx}
		~=~ 2\norm{\hat{\varphi}}\cdot \norm{\hat{\varphi}\cdot \mathbf{1}_{\geq \Delta}}.
		\]
		Since $\varphi$ is Fourier-concentrated, this is at most $2\epsilon(\Delta) \norm{\hat{\varphi}}^2 = 2\epsilon(\Delta) \norm{\varphi}^2 = 2\epsilon(\Delta)$, where we use the isometry of the Fourier transform and the assumption that $\norm{\varphi}^2 = \int \varphi^2(\bx)d\bx=1$. Plugging back the definition of $\Delta$, the result follows.
	\end{proof}
	
	\begin{lemma}\label{lem:doublesum}
		It holds that
		\[
		\sum_{z_1\neq z_2\in \mathbb{Z}}|a_{z_1}|\cdot|a_{z_2}|\cdot\epsilon(r|z_1-z_2|)~\leq~2\sum_{n=1}^{\infty}\epsilon(nr)
		\]
	\end{lemma}
	\begin{proof}
		For simplicity, define $\epsilon'(x)=\epsilon(x)$ for all $x>0$, and $\epsilon(0)=0$. Then the expression in the lemma equals
		\begin{align*}
		&\sum_{z_1,z_2\in \mathbb{Z}}|a_{z_1}|\cdot |a_{z_2}|\cdot\epsilon'(|z_1-z_2|r)\\
		&=
		\sum_{z_1,z_2\in \mathbb{Z}}\left(|a_{z_1}|\sqrt{\epsilon'(|z_1-z_2|r)}\right)\left(|a_{z_2}|\sqrt{\epsilon'(|z_1-z_2|r)}\right)\\
		&\leq
		\sqrt{\sum_{z_1,z_2\in \mathbb{Z}}|a_{z_1}|^2 \epsilon'(|z_1-z_2|r)}\sqrt{\sum_{z_1,z_2\in \mathbb{Z}}|a_{z_2}|^2\epsilon'(|z_1-z_2|r)}\\
		&=
		\sum_{z_1,z_2\in \mathbb{Z}}|a_{z_1}|^2 \epsilon'(|z_1-z_2|r)		
		\end{align*}
		where in the last step we used the fact that the two inner square roots are the same up to a different indexing. Recalling the definition of $\epsilon'$ and that $\sum_{z}|a_z|^2\leq 1$, the above is at most
		\begin{align*}
		&\sum_{z_1\in \mathbb{Z}}|a_{z_1}|^2\sum_{z_2\in \mathbb{Z}} \epsilon'(|z_1-z_2|r)~\leq~
		\sup_{z_1\in \mathbb{Z}}\sum_{z_2\in \mathbb{Z}}\epsilon'(|z_1-z_2|r)\\
		&= \left(\epsilon'(0)+2\sum_{n=1}^{\infty}\epsilon'(nr)\right)~=~
		2\sum_{n=1}^{\infty}\epsilon(nr)~.
		\end{align*}
	\end{proof}
	
	\begin{lemma}\label{lem:smallcorr}
		For any $g\in L^2(\reals^d)$, if $d\geq c'$ (for some universal constant $c'$), and we sample $\bw$ uniformly at random from $\{\bw:\norm{\bw}=2r\}$, it holds that
		\[
		\E\left[\left(\inner{g,\widehat{\psi_{\bw}\varphi}} - a_0\inner{g,\hat{\varphi}}\right)^2\right]~\leq~ 	
		10\norm{g}^2\left(\exp(-cd)+\sum_{n=1}^{\infty}\epsilon(nr)\right)
		\]
		where $a_0$ is the coefficient from \lemref{lem:psivarph} and $c$ is a universal positive constant.
	\end{lemma}
	\begin{proof}
		By symmetry, given any function $f$ of $\bw$, the expectation $\E_\bw[f(\bw)]$ (where $\bw$ is uniform on a sphere) can be equivalently written as $\E_{\bw\in\Wcal}\E_U [f(U\bw)]=\E_U \E_{\bw\in\Wcal}[f(U\bw)]$, where $U$ is a rotation matrix chosen uniformly at random (so that for any $\bw$, $U\bw$ is uniformly distributed on the sphere of radius $\|\bw\|$), and $\E_{\bw\in\Wcal}$ refers to a uniform distribution of $\bw$ over some finite set $\Wcal$ of vectors of norm $2r$. In particular, we will choose $\Wcal = \{\bw_1,\ldots,\bw_{\lceil \exp(cd)\rceil}\}$ (where $c$ is some positive universal constant) which satisfies the following:
		\begin{equation}\label{eq:Wcal}
		\forall i~\norm{\bw_i} = 2r~~,~~ \forall i\neq j~|\inner{\bw_i,\bw_j}|<2r^2.
		\end{equation}
		The existence of such a set follows from standard concentration of measure arguments (i.e. if we pick that many vectors uniformly at random from $\left\{-\frac{2r}{\sqrt{d}},+\frac{2r}{\sqrt{d}}\right\}^d$, and $c$ is small enough, the vectors will satisfy the above with overwhelming probability, hence such a set must exist). 
		
		Thus, our goal is to bound $\E_U\E_{\bw\in\Wcal}\E\left[\left(\inner{g,\widehat{\psi_{\bw}\varphi}} - a_0\inner{g,\hat{\varphi}}\right)^2\right]$. In fact, we will prove the bound stated in the lemma for any $U$, and will focus on $U=I$ without loss of generality (the argument for other $U$ is exactly the same). First, by applying \lemref{lem:psivarph}, we have
		\begin{align}
		\E_{\bw\in\Wcal}\left[\left(\inner{g,\widehat{\psi_{\bw}\varphi}} - a_0\inner{g,\hat{\varphi}}\right)^2\right]&~=~\E_{\bw\in\Wcal}\left[\left(\inner{g,\sum_{z\in\mathbb{Z}}a_z\hat{\varphi}(\cdot-z\bw)}-a_0\inner{g,\hat\varphi}\right)^2\right]\notag\\
		&~=~ \E_{\bw\in\Wcal}\left[\inner{g,\sum_{z\in\mathbb{Z}\setminus\{0\}}a_z\hat{\varphi}(\cdot-z\bw)}^2\right].\label{eq:smallcorr0}
		\end{align}
		For any $\bw\in\Wcal$, let 
		\[
		A_{\bw} = \{\bx\in \reals^d~:~\exists z\in \mathbb{Z}\setminus\{0\}~s.t.~\norm{\bx-z\bw}< r\}.
		\]
		In words, each $A_{\bw}$ corresponds to the union of open balls of radius $r$ around $\pm \bw, \pm 2\bw,\pm 3\bw\ldots$. An important property of these sets is that they are disjoint: $A_{\bw}\cap A_{\bw'}=\emptyset$ for any distinct $\bw,\bw'\in \Wcal$. To see why, note that if there was some $\bx$ in both of them, it would imply  $\norm{\bx-z_1\bw}<r$ and $\norm{\bx-z_2\bw'}<r$ for some non-zero $z_1,z_2\in \mathbb{Z}$, hence $\norm{z_1\bw-z_2\bw'}<2r$ by the triangle inequality. Squaring both sides and performing some simple manipulations (using the facts that $\norm{\bw}=\norm{\bw'}=2r$ and $|z_1|,|z_2|\geq 1$), we would get
		\[
		2|z_1 z_2|\cdot\left|\inner{\bw,\bw'}\right| > 4r^2(z_1^2+z_2^2-1) \geq 2r^2(z_1^2+z_2^2)~~\Rightarrow~~
			\left|\inner{\bw,\bw'}\right| \geq r^2\left(\left|\frac{z_1}{z_2}\right|+\left|\frac{z_2}{z_1}\right|\right) \geq 2r^2,
		\]
		where we used the fact that $x+1/x \geq 2$ for all $x>0$. This contradicts the assumption on $\Wcal$ (see \eqref{eq:Wcal}), and establishes that $\{A_{\bw}\}_{\bw\in\Wcal}$ are indeed disjoint sets.
		
		We now continue by analyzing \eqref{eq:smallcorr0}. Letting $\mathbf{1}_{A_{\bw}}$ be the indicator function to the set $A_{\bw}$, and $\mathbf{1}_{A^C_{\bw}}$ be the indicator of its complement, and recalling that $(a+b)^2\leq 2(a^2+b^2)$, we can upper bound \eqref{eq:smallcorr0} by
		\begin{equation}\label{eq:smallcorr1}
		2\cdot\E_{\bw\in\Wcal}\left[\inner{g,\mathbf{1}_{A_{\bw}}\sum_{z\in\mathbb{Z}\setminus\{0\}}a_z\hat{\varphi}(\cdot-z\bw)}^2\right]+2\cdot\E_{\bw\in\Wcal}\left[\inner{g,\mathbf{1}_{A^C_{\bw}}\sum_{z\in\mathbb{Z}\setminus\{0\}}a_z\hat{\varphi}(\cdot-z\bw)}^2\right].
		\end{equation}
		We consider each expectation separately. Starting with the first one, we have
		\begin{align*}
		\E_{\bw\in\Wcal}&\left[\inner{g,\mathbf{1}_{A_{\bw}}\sum_{z\in\mathbb{Z}\setminus\{0\}}a_z\hat{\varphi}(\cdot-z\bw)}^2\right]~=~
		\E_{\bw\in\Wcal}\left[\inner{g,\mathbf{1}_{A_{\bw}}\left(\widehat{\psi_{\bw}\varphi}-a_0\hat\varphi\right)}^2\right]\\
		&=\E_{\bw\in\Wcal}\left[\inner{\mathbf{1}_{A_{\bw}}g,\widehat{\psi_{\bw}\varphi}-a_0\hat\varphi}^2\right]
		~\leq~
		\E_{\bw\in\Wcal}\left[\norm{\mathbf{1}_{A_{\bw}}g}^2\norm{\widehat{\psi_{\bw}\varphi}-a_0\hat\varphi}^2\right]\\
		&\leq 2\cdot\E_{\bw\in\Wcal}\left[\norm{\mathbf{1}_{A_{\bw}}g}^2\left(\norm{\widehat{\psi_{\bw}\varphi}}^2+\norm{a_0\hat\varphi}^2\right)\right]~=~
		2\cdot \E_{\bw\in\Wcal}\left[\norm{\mathbf{1}_{A_{\bw}}g}^2\left(\norm{\psi_{\bw}\varphi}^2+|a_0|^2\cdot\norm{\hat\varphi}^2\right)\right].
		\end{align*}
		Since we have $\norm{\hat\varphi}=\norm{\varphi}=1$, $|a_0|^2\leq \sum_z |a_z|^2 \leq 1$, and $\norm{\psi_{\bw}\varphi}^2 = \int \psi^2_{\bw}(\bx)\varphi^2(\bx)d\bx \leq \int \varphi^2(\bx)d\bx=1$, the above is at most
		\begin{align*}
		4\cdot\E_{\bw\in\Wcal}\left[\norm{\mathbf{1}_{A_{\bw}}g}^2\right]&~\leq~
		\frac{4}{|\Wcal|}\sum_{\bw\in\Wcal}\int 1_{A_{\bw}}(\bx)|g(\bx)|^2d\bx\\
		&~=~
		\frac{4}{|\Wcal|}\int \left(\sum_{\bw\in\Wcal}1_{A_{\bw}}(\bx)\right)|g(\bx)|^2d\bx.
		\end{align*}
		Since $A_{\bw}$ are disjoint sets, $\sum_{\bw\in\Wcal}\mathbf{1}_{A_{\bw}}(\bx)\leq 1$ for any $\bx$, so the above is at most
		\begin{equation}\label{eq:smallcorr2}
		\frac{4}{|\Wcal|}\int |g(\bx)|^2d\bx ~\leq~ 4\exp(-cd)\norm{g}^2. 
		\end{equation}
		
		We now turn to analyze the second expectation in \eqref{eq:smallcorr1}, namely $\E_{\bw\in\Wcal}\left[\inner{g,\mathbf{1}_{A^C_{\bw}}\sum_{z\in\mathbb{Z}\setminus\{0\}}a_z\hat{\varphi}(\cdot-z\bw)}^2\right]$. We will upper bound the expression deterministically for any $\bw$, so we may drop the expectation. Applying Cauchy-Schwartz, it is at most
		\begin{equation}
		\norm{g}^2\cdot \norm{\mathbf{1}_{A^C_{\bw}}\sum_{z\in\mathbb{Z}\setminus\{0\}}a_z\hat{\varphi}(\cdot-z\bw)}^2~=~ \norm{g}^2\left(\sum_{z_1,z_2\in \mathbb{Z}\setminus\{0\}}a_{z_1}\overline{a_{z_2}}\inner{\mathbf{1}_{A^C_{\bw}}\hat{\varphi}(\cdot-z_1\bw),\hat{\varphi}(\cdot-z_2\bw)}\right).\label{eq:smallcorr3}
		\end{equation}
		We now divide the terms in the sum above to two cases:
		\begin{itemize}
			\item If $z_1=z_2$, then
		\[
		\inner{\mathbf{1}_{A^C_{\bw}}\hat{\varphi}(\cdot-z_1\bw),\hat{\varphi}(\cdot-z_2\bw)} ~=~ \int \mathbf{1}_{A^C_{\bw}}(\bx)|\hat{\varphi}(\bx-z_1\bw)|^2d\bx.
		~=~ \int \mathbf{1}_{A^C_{\bw}}(\bx+z_1\bw)|\hat{\varphi}(\bx)|^2d\bx,
		\]
		and by definition of $A_{\bw}^C$ and the assumption $z_1\neq 0$, we have $\mathbf{1}_{A_{\bw}^C}(\bx+z_1\bw)=1$ only if $\norm{\bx}\geq r$. Therefore, as $\varphi$ is Fourier-concentrated, the above is at most
		\[
		\int_{\bx:\norm{\bx}\geq r}|\hat{\varphi}(\bx)|^2d\bx ~\leq~ \epsilon^2(r)\cdot \norm{\hat{\varphi}}^2 ~=~ \epsilon^2(r)\cdot \norm{\varphi}^2 ~=~\epsilon^2(r).
		\]
			\item If $z_1\neq z_2$, then by \lemref{lem:compind}, 
			\[
			\inner{\mathbf{1}_{A^C_{\bw}}\hat{\varphi}(\cdot-z_1\bw),\hat{\varphi}(\cdot-z_2\bw)} ~\leq~ \inner{|\hat{\varphi}(\cdot-z_1\bw)|,|\hat{\varphi}(\cdot-z_2\bw)|}
			~\leq~ 2\epsilon(|z_1-z_2|r). 
			\]
		\end{itemize}
		Plugging these two cases back into \eqref{eq:smallcorr3}, we get the upper bound
		\[
		\norm{g}^2\left(\sum_{z\in \mathbb{Z}\setminus\{0\}}|a_z|^2\epsilon^2(r)+2\sum_{z_1\neq z_2\in \mathbb{Z}}|a_{z_1}|\cdot|a_{z_2}| \cdot\epsilon(|z_1-z_2|r)\right).
		\]
		Noting that $\sum_z |a_z|^2\leq 1$, and applying \lemref{lem:doublesum}, the above is at most
		\[
		\norm{g}^2\left(\epsilon^2(r)+4\sum_{n=1}^{\infty}\epsilon(nr)\right)~\leq~
		5\norm{g}^2\sum_{n=1}^{\infty}\epsilon(nr),
		\]
		where we used the fact that $\epsilon^2(r)\leq \epsilon(r)\leq \sum_{n=1}^{\infty}\epsilon(nr)$. 
		Recalling this is an upper bound on the second expectation in \eqref{eq:smallcorr1}, and that the first expectation is upper bounded as in \eqref{eq:smallcorr2}, we get that \eqref{eq:smallcorr1} (and hence the expression in the lemma statement) is at most
		\[
		10\norm{g}^2\left(\exp(-cd)+\sum_{n=1}^{\infty}\epsilon(nr)\right)
		\]
		as required.
	\end{proof}

	With these lemmas in hand, we can now turn to prove the theorem. We have that
	\[
	\var_{\bw^\star}\left[\nabla F_{\bw^\star}(\bw)\right]~=~\E_{\bw^\star}\norm{\nabla F_{\bw^\star}(\bw)-\E_{\bw^\star}[\nabla F_{\bw^\star}(\bw)]}^2~\leq~ \E_{\bw^\star}\norm{\nabla F_{\bw^\star}(\bw)-\bp}^2
	\]
	for any vector $\bp$ which is not dependent of $\bw^\star$ (this $\bp$ will be determined later). Recalling the definition of the objective function $F$, and letting $\bg(\bx) = (g_1(\bx),g_2(\bx),\ldots) = \frac{\partial}{\partial \bw} f(\bw,\bx)$, the above equals
	\begin{align*}
	\E_{\bw^\star}&\norm{\E_{\bx\sim\varphi^2}\left[\left(f(\bw,\bx)-\psi(\inner{\bw^\star,\bx})\right)\bg(\bx)\right]-\bp}^2\\
	&=\sum_i\E_{\bw^\star}\left(\E_{\bx\sim\varphi^2}\left[f(\bw,\bx)g_i(\bx)-\psi(\inner{\bw^\star,\bx})g_i(\bx)\right]-p_i\right)^2\\\
	&=\sum_i\E_{\bw^\star}\left(\inner{\varphi g_i,\varphi f(\bw,\cdot)}-\inner{\varphi g_i,\varphi\psi_{\bw^\star}}-p_i\right)^2
	\end{align*}
	Let us now choose $\bp$ so that $p_i = \inner{\varphi g_i,\varphi f(\bw,\cdot)}-\inner{\varphi g_i,a_0 \varphi}$ (note that this choice is indeed independent of $\bw^\star$). Plugging back and applying \lemref{lem:smallcorr} (using the $L^2$ function $\widehat{\varphi g_i}$ for each $i$), we get
	\begin{align*}
	\sum_i\E_{\bw^\star}\left(\inner{\varphi g_i,\varphi\psi_{\bw^\star}}-\inner{\varphi g_i,a_0 \varphi}\right)^2
	&~=~
	\sum_i\E_{\bw^\star}\left(\inner{\widehat{\varphi g_i},\widehat{\varphi\psi_{\bw^\star}}}-\inner{\widehat{\varphi g_i},a_0 \widehat{\varphi}}\right)^2	\\
	&~\leq~ 10\sum_i\norm{\widehat{\varphi g_i}}^2\left(\exp(-cd)+\sum_{n=1}^{\infty}\epsilon(nr)\right),
	\end{align*}
	and since
	\[
	\sum_{i=1}^{d}\norm{\varphi g_i}^2 ~=~ \sum_{i=1}^{d}\int g_i^2(\bx)\varphi^2(\bx)d\bx ~=~ \int \norm{\bg(\bx)}^2\varphi^2(\bx)~=~
	\E_{\bx\sim\varphi^2}\norm{\bg(\bx)}^2 ~\leq~ G_{\bw}^2,
	\]
	the theorem follows.

	\subsection{Proof of \thmref{thm:trajectory}}
	We will assume w.l.o.g. that the algorithm is deterministic: If it is 
	randomized, we can simply prove the statement for any possible realization 
	of its random coin flips.
	
	We consider an oracle which given a point $\bw$, returns 
	$\E_{\bw^\star}[\nabla F_{\bw^\star}(\bw)]$ if $|\nabla 
	F_{\bw^\star}(\bw)-\E_{\bw^\star}[\nabla F_{\bw^\star}(\bw)]|\leq 
	\varepsilon$, and $\nabla 
	F_{\bw^\star}(\bw)$ otherwise. Thus, it is enough to show that with 
	probability at least $1-p$, the oracle will only return responses of the 
	form $\E_{\bw^\star}[\nabla F_{\bw^\star}(\bw)]$, which is clearly 
	independent of $\bw^\star$. Since the algorithm's output can depend on 
	$\bw^\star$ only through the oracle responses, this will prove the theorem.
	
	The algorithm's first point $\bw_1$ is fixed before receiving any 
	information from the oracle, and is therefore independent of $\bw^\star$. 
	By \thmref{thm:main}, we have that 
	$\var_{\bw^\star}(\nabla F_{\bw^\star}(\bw_1))\leq \varepsilon$, which by 
	Chebyshev's inequality, implies that 
	\[
	\Pr\left(|\nabla F_{\bw^\star}(\bw_1)-\E_{\bw^\star}[\nabla 
	F_{\bw^\star}(\bw_1)]|>\varepsilon\right)~\leq~ \varepsilon,
	\]
	where the probability is over the choice of $\bw^\star$. 
	Assuming the event above does not occur, the oracle returns 
	$\E_{\bw^\star}[\nabla F_{\bw^\star}(\bw)]$, which does not depend on the 
	actual choice of $\bw^\star$. This means that the next point $\bw_2$ chosen 
	by the algorithm is fixed independent of $\bw^\star$. Again by 
	\thmref{thm:main} and Chebyshev's 
	inequality,
	\[
	\Pr\left(|\nabla F_{\bw^\star}(\bw_2)-\E_{\bw^\star}[\nabla 
	F_{\bw^\star}(\bw_2)]|>\varepsilon\right)~\leq~ \varepsilon.
	\]
	Repeating this argument and applying a union bound, it follows that as long 
	as the number of iterations $T$ satisfies $T\varepsilon\leq p$ (or 
	equivalently $T\leq p/\varepsilon$), the 
	oracle reveals no information whatsoever on the choice of $\bw^\star$
	all point chosen by the algorithm (and hence also its output) are 
	independent of $\bw^\star$ as required.

%

	\subsubsection*{Acknowledgements}
	This research is supported in part by an FP7 Marie Curie CIG grant, Israel Science Foundation grant 425/13, and the Intel ICRI-CI Institute.

	\bibliographystyle{plain}
	\bibliography{mybib}
	
\end{document}